\documentclass[11pt]{article}
\usepackage{amssymb}
\usepackage{amsmath}
\usepackage{amsbsy}
\usepackage{subfigure}
\usepackage{natbib}
\usepackage{fullpage}

\RequirePackage[colorlinks,citecolor=blue,urlcolor=blue,breaklinks]{hyperref}
\newtheorem{example}{Example} 
\newtheorem{theorem}{Theorem}
\newtheorem{lemma}[theorem]{Lemma} 
 
\newtheorem{remark}[theorem]{Remark}
\newtheorem{corollary}[theorem]{Corollary}

\def\proclaim#1{\par \bigskip\noindent {\bf #1}\bgroup\it\ }
\def\endproclaim{\egroup\par\bigskip}

\def\text#1{\mbox{\rm #1}}

\def\z{\boldsymbol{z}}

\newcommand{\norm}[1]{\left\|#1\right\|}
\newcommand{\bracket}[1]{\left(#1\right)}
\newcommand{\Bracket}[1]{\left[#1\right]}
\newcommand{\bbracket}[1]{\left\{#1\right\}}
\newcommand{\abs}[1]{\left|#1\right|}
\newcommand{\mR}{\mathbb{R}}
\newcommand{\mE}{\mathbb{E}}
\newcommand{\mP}{\mathbb{P}}

\usepackage{graphicx}
\usepackage{appendix}

\newcommand{\bfb}{b}

\newcommand{\batch}{B}
\newcommand{\bba}{A}

\newcommand{\bbv}{V}

\newcommand{\cO}{\mathcal{O}}

\makeatletter
\newcommand\footnoteref[1]{\protected@xdef\@thefnmark{\ref{#1}}\@footnotemark}
\makeatother

\usepackage{xcolor}
\definecolor{DSgray}{cmyk}{0,1,0,0}

\begin{document}
\title{Acceleration of stochastic gradient descent with momentum by averaging:  finite-sample rates and asymptotic normality}
\author{Kejie Tang \footnote{School of Mathematical Sciences, Shanghai Jiao Tong University}\quad Weidong Liu \footnote{School of Mathematical Sciences, MoE Key Lab of Artificial Intelligence, Shanghai Jiao Tong University}\quad Yichen Zhang\footnote{Daniels School of Business, Purdue University}\quad Xi Chen\footnote{Stern School of Business, New York University}}
\date{}
\maketitle
    
\begin{abstract}
Stochastic gradient descent with momentum (SGDM) has been widely used in many machine learning and statistical applications. Despite the observed empirical benefits of SGDM over traditional SGD, the theoretical understanding of the role of momentum for different learning rates in the optimization process remains widely open. We analyze the finite-sample convergence rate of SGDM under the strongly convex settings and show that, with a large batch size, the mini-batch SGDM converges faster than the mini-batch SGD to a neighborhood of the optimal value. 
Additionally, our findings, supported by theoretical analysis and numerical experiments, indicate that SGDM permits broader choices of learning rates.
Furthermore, we analyze the Polyak-averaging version of the SGDM estimator, establish its asymptotic normality, and justify its asymptotic equivalence to the averaged SGD. The asymptotic distribution of the averaged SGDM enables uncertainty quantification of the algorithm output and statistical inference of the model parameters. 
\end{abstract}

\newpage

\section{Introduction}

In this paper, we are interested in solving the following 
 stochastic optimization problem:
\begin{eqnarray}\label{so}
x^{*}=\text{argmin}_{x\in \mR^{d}}\mE [ f_{\xi}(x)],
\end{eqnarray}
where $f_{\xi}(x)$ is a stochastic convex loss function, $\xi$ is a random element, and $\mE [\cdot]$ is the expectation respected to $\xi$. Stochastic optimization plays an important role in many statistics and machine learning problems. For a wide range of applications, the objective function is strongly convex, and stochastic gradient descent (SGD) is often preferred over gradient descent (GD) (see, e.g., \citealp{nesterov2003introductory,nocedal2006numerical}) due to its computational advantage. SGD is a first-order optimization algorithm that approximates the expected loss by averaging the loss function over a mini-batch of training examples. At each iteration, the algorithm updates the model parameters in the direction of the negative gradient of the mini-batch loss, scaled by a learning rate parameter. 

While SGD is simple and easy to implement, it may suffer from slow convergence rates or oscillations in high-dimensional optimization problems, particularly when the loss function is noisy and ill-conditioned. Momentum-based methods enhance SGD by introducing an exponentially weighted moving average of the past gradients to the update rule, which serves to dampen oscillations and accelerate convergence, allowing the algorithm to maintain a more consistent direction of movement even in the presence of noisy gradients. 
SGD with momentum (SGDM) has become increasingly popular in modern applications, e.g., large-scale deep neural networks. 
Evident numerical studies have shown that the use of momentum-based optimization methods improves the convergence rate and the generalization performance, as well as reduces the sensitivity to the choice of hyperparameters. However, the theoretical analysis of SGDM is still an active area of research.

In this paper, we focus on the theoretical analysis from an optimization perspective under strong convexity that ensures the existence of a unique and well-defined global minimizer. 
Our goal is to establish the convergence properties of SGDM to provide insights into the role of momentum and other hyperparameters in the optimization process.
For solving (\ref{so}), SGDM updates the target estimator by a weighted combination with historical gradients 
\begin{eqnarray}\label{ab}
m_{t+1}&=&\gamma m_{t}+(1-\gamma)\triangledown f_{\xi_{t}}(x_{t}),\cr
x_{t+1}&=&x_{t}-\alpha m_{t+1},
\end{eqnarray}
where $\xi_{t}$, $t\geq 1$ are i.i.d. samplings from $\xi$, $\gamma$ is the momentum weight and $\alpha$ is the learning rate. The classical SGD is a special case of SGDM with  $\gamma=0$ and $m_{t+1}=\nabla f_{\xi_t}(x_t)$. 

In practice, a large momentum weight is often placed to accelerate the algorithm, for example, $\gamma=0.9$. That being said, it is still an open problem to investigate the theoretical properties of SGDM with a general specification of $\gamma$. The theoretical analysis in existing studies has not given an affirmative answer to the open problem by \cite{liu2020improved} and the assertion
by \cite{kidambi2018insufficiency}.
In this paper, we consider the mini-batch SGDM with batch size $\batch$. 
We establish finite sample rates for SGDM and the averaging of its trajectories under smooth and strongly convex loss functions. By our results, we can give some partial answers to these open questions and some related contributions in several aspects.

\begin{enumerate}
\item Drawing upon the assertion presented in \cite{kidambi2018insufficiency}, this study addresses the open question posited by \cite{liu2020improved} by demonstrating that mini-batch SGDM, with appropriate momentum weights, converges to a local neighborhood of the minimum with a faster rate than SGD. We rigorously establish the finite-sample convergence rate, and we further provide an adaptive choice of the momentum weight which theoretically attains the optimal convergence rate. Our experiments support the theoretical results of convergence and the optimal momentum weight.

\item We provide a non-asymptotic analysis of the Polyak-Ruppert averaging version of SGDM. The averaging of SGDM trajectories $\bar{x}_{t}=\frac{1}{t-n_{0}}\sum_{i=n_{0}+1}^{t}x_{i}$ can accelerate SGDM with a wide range of learning rates $\alpha$ to the
rate $\cO(\frac{\sigma}{\sqrt{t}})$, where $\sigma^{2}$ is the variance of stochastic gradient. Furthermore, we show theoretically that averaged SGDM converges faster than averaged SGD in the early phases of the iteration process. Moreover, through both theoretical analysis and numerical experiments, we demonstrate that SGDM is less sensitive to the choice of learning rates, and in addition to this, the averaged SGDM is less sensitive to the start of the averaging iteration $n_0$.

\item We further establish the asymptotic normality for the averaged $\bar{x}_{t}$ as $t\rightarrow \infty$ with a decaying learning rate $\alpha$ or a diverging batch size $\batch$. Particularly, as $\alpha=\Theta(t^{-\epsilon})$ for $\epsilon\in(\frac{1}{2},1)$, the asymptotic normality holds for averaged SGDM under strongly convex loss functions. The asymptotic covariance matrix depends only on the Hessian and Gram matrix of the stochastic gradients at $x^*$, and the batch size $\batch$. Interestingly, mini-batch averaged SGDM is asymptotically equivalent to mini-batch averaged SGD as $t\rightarrow \infty$. We further demonstrate that the optimal learning rates of averaged SGDM that correspond to asymptotic normality are $\alpha = \Theta(t^{-1/2})$ for quadratic losses and $\alpha = \Theta(t^{-2/3})$ for general strongly convex losses. To our best knowledge, this is the first work to analyze the convergence of averaged SGDM to asymptotic normality under mini-batching. The results enable us to perform uncertainty quantification for the algorithm outputs of the averaged SGDM algorithm $\bar x_t$, and statistical inference for model parameters $x^*$. 
\end{enumerate}

In addition to the convergence rate under pre-specified learning rates, the selection of appropriate learning rates is a critical aspect when evaluating the performance of optimizers, as either an excessively high or low learning rate can detrimentally affect convergence. We study the {\textit{sensitivity}} of the convergence over different learning rates. Extensive research has been conducted to tackle this challenge and accelerate convergence, such as \cite{zeiler2012adadelta, kingma2014adam}. Recent results \citep{paquette2021dynamics,bollapragada2022fast} demonstrated the acceleration of SGDM on quadratic forms, however, under a shrinking range of allowable learning rates for larger momentum. This finding contradicts the advantages of SGDM and indicates an increased sensitivity to the choice of learning rates in theoretical analysis.

\subsection{Related Works}\label{subsec:related}
Since the seminal work by \cite{robbins1951stochastic}, the convergence properties of SGD  have been extensively studied in the literature (See, e.g., \citealp{moulines2011non, bottou2018optimization,nguyen2018sgd}). 
On the contrary, the theory for SGDM has rarely been explored until recently, although it is very popular in training many modern machine learning models such as neural networks to improve the training speed and accuracy of various models.

\cite{kidambi2018insufficiency} showed that (\ref{ab}) cannot achieve any
improvement over SGD for a specially constructed linear regression problem. Together with some numerical experiments,  they asserted that the only reason for the superiority of stochastic momentum methods in
practice is mini-batching. In a recent paper, \cite{liu2020improved} proved SGDM can be as fast as  SGD. They established the identical convergence bound as SGD. They also posed an
open problem of whether it is possible to show that
SGDM converges faster than SGD for special objectives such as quadratic ones. Other studies on the theoretical analysis of SGDM can be found in 
\cite{loizou2017linearly}, \cite{loizou2020momentum}, \cite{gitman2019understanding} for linear system and quadratic loss; \cite{sebbouh2021almost}
for almost sure convergence rates under smooth and convex loss functions with a time-varying momentum weight. \cite{mai2020convergence} studied a class of general convex loss and obtained convergence rates of time averages regardless of the momentum weight. We refer the readers to the latter paper and \cite{liu2020improved} for a few more papers on the convergence rate for SGDM. \cite{richtarik2020stochastic} considered the problem of solving a consistent linear system $Ax=b$ on least square regression. SGDM is also called the stochastic heavy-ball method, originated from Polyak’s heavy-ball method for deterministic optimization \citep{polyak1964some}. Some works \citep{loizou2017linearly,kidambi2018insufficiency,loizou2020momentum,paquette2021dynamics, bollapragada2022fast,lee2022trajectory} based on the stochastic heavy ball method (SHB) achieved the linear convergence rate. \cite{yang2016unified, defazio2020understanding, jin2022convergence,li2022last} investigated the impact of momentum on the convergence properties of non-convex optimization problems and provided insights into its practical use. \cite{gitman2019understanding} focused on the noise reduction properties of momentum. 

The averaged SGDM is not widely analyzed in the literature, but averaging tools have been studied for SGD and its variants since \cite{ruppert1988efficient,polyak1992acceleration,moulines2011non} to accelerate the convergence and establish the asymptotic normality.
Building statistical inference and uncertainty quantification of SGD iterates has been an emerging topic, and an extensive list of literature follows, including \cite{chen2020statistical,su2023higrad,zhu2023online} that studied averaged SGD and provided inference procedures based on plug-in, batch-means, and tree-based construction of the confidence intervals, respectively. \cite{zhu2021constructing,lee2022fast} applied process-level function central limit theorem and utilized it to construct confidence regions seamlessly. \cite{toulis2021proximal,chen2023online} studied the uncertainty quantification of the implicit and gradient-free variants of the SGD procedures. 

\section{Preliminaries}\label{sec:pre}
In this section, we present the mini-batch SGDM settings considered in this paper.  Particularly, we define a mini-batch stochastic loss
\begin{eqnarray}\label{so2}
    g_{\eta_{t}}(x)=\frac{1}{\batch}\sum_{i=1}^{\batch}f_{\xi_{ti}}(x),    
\end{eqnarray}
where $\eta_{t}=\{\xi_{t1},...,\xi_{t\batch}\}$ is a mini-batch of size $\batch$ and elements $\xi_{ti}$ are i.i.d. sampled from the distribution of $\xi$. The update of $m_{t}$ in (\ref{ab}) uses the mini-batch stochastic gradient $\triangledown g_{\eta_t}(x_t)$ instead of the individual stochastic gradient $\triangledown f_{\xi_{ti}}(x_t)$.

We first introduce some regularity assumptions and discuss their use in the theoretical results. We denote the $\ell_2$-norm of a vector as $\norm{\cdot}$ and the operator norm as $\norm{A} = \max_{\norm{x}=1} \norm{Ax}$. 

\noindent{\bf(A1).}  Assume that the loss function $f_{\xi}(x)$ is twice differentiable. Let $\kappa_{1}\leq \cdots\leq \kappa_{d}$  be the eigenvalues of  Hessian matrix $\Sigma=\mE[\triangledown^{2} f_{\xi}(x^{*})]$. Assume that
$\mu:=\kappa_{1}>0$, and $L:=\kappa_{d}<\infty.$

\noindent{\bf(A2).} There exists a constant $\overline{L}\geq 0$ such that the Hessian matrix of the loss $f(x):= \mE[f_\xi(x)]$, $\Sigma(x)=\triangledown^{2} f(x)$  satisfies  $\norm{\Sigma(x)-\Sigma(x^*)}\leq \overline{L}\norm{x-x^*}$, $ \forall x$. 

\noindent{\bf (A3).} Define $L_\xi = \sup_x \frac{\norm{\nabla f_\xi(x)-\nabla f_\xi(x^*)}}{\norm{x-x^*}}$. Assume that $
\mE[ \norm{\triangledown f_{\xi}(x^{*})}^{2}]\leq \sigma^{2} $ and  $\mE[L_\xi^2]\leq L_f^2$. 

\noindent{\bf (A3').} Assume that $\triangledown f_{\xi}(x^{*})$ and $L_\xi$ satisfy 
\begin{eqnarray*}
    \sup_{\norm{v}=1}\mE\Bracket{\exp\bracket{\abs{v^\top \triangledown f_{\xi}(x^{*})}/\sigma}} \leq 2, \quad \mE\Bracket{\exp\bracket{L_\xi/L_f}} \leq 2.
\end{eqnarray*}

A few discussions follow concerning the assumptions above. 
Assumption (A1) is a regularity condition on the smoothness and strong convexity of the loss function $f_{\xi}(x)$ at $x^*$. 
Beyond that, (A2) assumes the Lipschitz condition on the Hessian matrix of the loss function. Notably, the $\Sigma$ defined in (A1) is a brief notation for $\mE[\Sigma_{\xi}(x^*)]$ defined in (A2). For quadratic losses, the Hessian matrix $\Sigma(x)$ is identical for different $x$, and therefore (A2) holds with $\overline L=0$. For general strongly convex losses, we illustrate (A2) under a logistic regression setting in Example \ref{ex:logistic} below. In the following, we consider two scenarios, quadratic losses, and general (strongly convex) losses. 

Assumptions (A3) and (A3') are two separate conditions on the smoothness of quadratic and general stochastic loss functions, respectively. 
Assumption (A3) is weaker than (A3') and indeed weaker than many in the literature, such as those in \cite{yan2018unified,liu2020improved} that $    \mE[\norm{\nabla f_\xi(x)-\nabla f(x)}^2]\leq \sigma^2$ for all $x$, which does not hold for linear regression problems with unbounded domain of $x$. On the other hand, \cite{bottou2018optimization,wang2022uniform} assumed that  $ \mE[\norm{\nabla f_\xi(x)-\nabla f(x)}^2]\leq \rho \norm{\nabla f(x)}^2+\sigma^2$ for $\rho>0$. Meanwhile, our assumptions lead to
$
    \mE[\norm{\nabla f_\xi(x)-\nabla f(x)}^2]\leq \frac{3(L_f^2+L^2) }{\mu^2} \norm{\nabla f(x)}^2+3\sigma^2,
$
due to the triangle inequality and strongly convexity.
For any bounded domain $\{\norm{x}\leq r\}$, (A3) and (A3') are easily satisfied for bounded gradients and smoothness. For unbounded problems, they are usually satisfied given certain design properties in many popular statistical models. 
For instance, we illustrate the assumptions under logistic regression in the following example. 

\begin{example}\label{ex:logistic}
Consider an $\ell_2$-regularized logistic regression model with samples $\{a_\xi, b_\xi\}$ such that $a_\xi\in\mR^d$ and $b_\xi\in\{0,1\}$ are generated by $b_\xi=1$ with probability $p_x(a_\xi)$ and $b_\xi=0$ otherwise,  where $p_x(a) = 1/(1+\exp(-x^\top a))$. We consider the loss function defined as:
\begin{eqnarray*}
    f_\xi(x) = -b_\xi \log(p_x(a_\xi))-(1-b_\xi) \log(1-p_x(a_\xi))+\frac{\nu}{2} \norm{x}^2. 
\end{eqnarray*}
The gradient and the Hessian matrix are 
\begin{eqnarray*}
    \nabla f_\xi(x) = (p_x(a_\xi)-b_\xi) a_\xi+\nu x,\quad
    \Sigma(x) = \mE[p_x(a_\xi)(1-p_x(a_\xi)) a_\xi a_\xi^\top]+\nu I_d.
\end{eqnarray*}
By the fact that $\abs{p_{x}(a)-p_y(a)}\leq \frac{\sqrt{3}}{18} \norm{a} \norm{x-y}$,  (A2) is satisfied with 
$
    \overline{L} =  \frac{\sqrt{3}}{6}\mE\norm{a_\xi}^3+\nu.
$
Suppose $\{a_\xi\}$ in logistic regression satisfy $\mE[a_\xi] = 0_d$ and $\sup_{\norm{v}=1}\mE\big[\exp\big(\abs{v^\top a_\xi}^2/\sigma^2\big)\big] \leq 2$, 
we have that $\abs{p_x(a)-b}\leq 2$, $\norm{(p_{x}(a)-p_y(a))a}\leq \frac{\sqrt{3}}{18} \norm{a}^2 \norm{x-y}$, and therefore,
\begin{eqnarray*}
    \sup_{\norm{v}=1}\mE\Bracket{\exp\bracket{\abs{v^\top \triangledown f_\xi(x^*)}/(c\sigma)}}  \leq 2, 
    ~ \mE\Bracket{\exp\bracket{L_\xi/(cL_f)}} \leq 2.
\end{eqnarray*}
for some absolute constant $c>0$ and $L_f=\mE\norm{a_\xi}^2+\nu$. 
\end{example}


\section{Finite-sample Convergence Rates for SGDM}
\label{sec:finite}
In this section, we first present the finite-sample convergence results for SGDM with general momentum weight $\gamma$. 
We consider the two cases separately: $\overline{L}=0$ corresponds to the quadratic losses, and $\overline{L}>0$ corresponds to general strongly convex losses.
\subsection{The finite-sample rates for SGDM on quadratic losses}
We first establish the finite-sample rates under quadratic losses, where 
\begin{eqnarray}\label{eq:quadratic}
f_{\xi}(x)=\frac{1}{2}x^\top A_{\xi}x-b^\top_{\xi}x+c.
\end{eqnarray}
The Hessian matrix of the loss function $\Sigma=\mE[A_{\xi}]$. The stochastic loss is $g_{\eta_{t}}(x)=\frac{1}{\batch}\sum_{i=1}^{\batch}f_{\xi_{ti}}(x)$ in the mini-batch setting. 
The following theorem shows the convergence rate of the last iterate $x_{t+1}$.

\begin{theorem}\label{thm:g_gamma}
    Under (A1)-(A3) and $\overline{L}=0$, for any momentum $\gamma\in[0,1)$ and fixed $\delta\in(0,1]$, assume the learning rate  $\alpha>0$ satisfies $\alpha L<2(1+\gamma)/(1-\gamma)$ and $16 M^2 \alpha^2 L_f^2\leq \batch\delta\lambda^{2(1-\delta)}(1-\lambda)$, where
\begin{eqnarray}
\label{def:m}M = \frac{4}{\sqrt{\Delta}}  \bracket{2(1-\gamma)(1+\alpha L+ L)+3\alpha\gamma},~\Delta= \min_k\bbracket{\abs{\bracket{\gamma+1-\alpha(1-\gamma)\kappa_k}^2-4\gamma}}>0,
\end{eqnarray}
    and $\lambda$ is the spectral radius of the matrix
    \begin{eqnarray}\label{matrix}
\Gamma = \left(
    \begin{array}{cc}
       \gamma I, & (1-\gamma)\Sigma \\
      -\alpha \gamma I,  & I-\alpha(1-\gamma)\Sigma \\
    \end{array}
\right).
\end{eqnarray}
    Let $\widetilde{m}_{t+1} =(1-\gamma) \sum_{j=1}^{t}\gamma^{t-j} \Sigma (x_j-x^*)$, we have for $t\geq 1$,
    \begin{eqnarray}\label{eq:last-it}
        \mE[\norm{\widetilde{m}_{t+1}}^2+\norm{{x}_{t+1}-x^*}^2]\leq 2M^2 \bracket{ \frac{4}{\batch(1-\lambda)}\alpha^2\sigma^2 +\norm{x_1-x^*}^2 \lambda^{2(1-\delta)t}}.
    \end{eqnarray}
\end{theorem}
Theorem \ref{thm:g_gamma} provides a finite-sample bound simultaneously for the error of the last iterate $x_{t+1}-x^*$, and a weighted average of the stochastic gradients $\widetilde m_{t+1}$. The first term in the error bound \eqref{eq:last-it}, $ \frac{8M^2}{\batch(1-\lambda)}\alpha^2\sigma^2 $, corresponds to a non-decaying bias, due to the noisy observation in the stochastic gradient. The bias term is proportional to the squared learning rate $\alpha^2$ and the variance of the stochastic gradient $\sigma^2$, which is the same as the one in SGD. The second term is exponentially decaying when $\lambda<1$, and establishes the convergence of the SGDM algorithm from any initialization $x_1$ to the true solution $x^*$. 
The linear convergence convergence is up to a neighborhood of $x^*$
with size $\sqrt{\frac{8M^2}{\batch(1-\lambda)}}\alpha\sigma$. With a larger batch size $B$ or a smaller $\sigma$, the size of the neighborhood will be smaller, which aligns with the experimental findings reported in \cite{kidambi2018insufficiency} that the superiority of momentum methods is mainly due to mini-batching.  Meanwhile, the second term in \eqref{eq:last-it}, $\norm{x_1-x^*}^2 \lambda^{2(1-\delta)t}$, remains important to determine the convergence rate in the initial stage. 
Subsequently, we will demonstrate that this term, particularly the quantity $\lambda$ in SGDM, is improved compared to the $\lambda$ in SGD.

\subsubsection{Linear convergence to a local neighborhood}
The second term in \eqref{eq:last-it} determines a linear convergence of SGDM to a local neighborhood of $x^*$ determined by its first term. The rate of this linear convergence is determined by $\gamma$, the spectral radius of the matrix \eqref{matrix}. We first provide some intuition in deriving $\lambda$. 
Considering the noiseless setting where the stochastic gradient is the same as the true gradient, i.e., $\nabla f_\xi(x)= \nabla f(x)$ for any $ x$ and $\xi$. 
We have $\nabla g (x) = \Sigma (x-x^*)$ and the SGDM updating rule (\ref{ab}) can be rewritten as 
\begin{eqnarray*}
    \left(\begin{array}{c}
         m_{t+1}\\
         x_{t+1}-x^*
    \end{array}\right) = \Gamma \left(\begin{array}{c}
         m_{t}\\
         x_{t}-x^*
    \end{array}\right) = \Gamma^t \left(\begin{array}{c}
         m_{1}\\
         x_{1}-x^*
    \end{array}\right),
\end{eqnarray*}
where $\Gamma$ is the matrix in (\ref{matrix}).
When the spectral radius of $\Gamma$ is less than 1, the full-batch gradient descent with momentum enjoys linear convergence. In Theorem \ref{thm:g_gamma}, the matrix $\Gamma$ satisfies $\norm{\Gamma^t}\leq M \lambda^t$ with $\lambda<1$, which is proved in the appendix. 
\begin{remark}
The assumption that $\Delta >0$ in \eqref{def:m} is placed for the diagonalization of the matrix $\Gamma$, which is required in our analysis to provide a last-iterate convergence analysis and can be relaxed if we aim for a time-average convergence analysis of the sum $\sum_t \mE[\norm{\widetilde{m}_{t}}^2+\norm{x_t-x^*}^2]$. 
More particularly, a time-average convergence analysis computes $\norm{\sum_{t=0}^\infty \Gamma^t} =\norm{(I-\Gamma)^{-1}}$ and requires only that the spectral radius $\lambda<1$.
\end{remark}
In Theorem \ref{thm:g_gamma}, the convergence rate 
mainly depends on the spectral radius $\lambda$ of the matrix \eqref{matrix}. 
We characterize its explicit form in the following theorem.

\begin{theorem}\label{thm:lambda}
For any momentum weight $\gamma\in[0,1)$, let $\alpha, \gamma$ satisfy $\alpha L<2(1+\gamma)/(1-\gamma)$ and define $\phi=\min\bbracket{\alpha\mu,2(1+\gamma)/(1-\gamma)-\alpha L}.$ If the momentum $\gamma<(1-\phi)^2/(1+\phi)^2$, we have that the spectral radius $\lambda$ of $\Gamma$ defined by (\ref{matrix}) satisfies
\begin{eqnarray}\label{eq:lambda}
    \lambda = \frac{\gamma+1-(1-\gamma)\phi + \sqrt{\bracket{\gamma+1-(1-\gamma)\phi}^2-4\gamma}}{2}.
\end{eqnarray}
On the other hand, if the momentum $\gamma\geq (1-\phi)^2/(1+\phi)^2$, we have $$\lambda=\sqrt{\gamma}.$$
\end{theorem}
Theorem \ref{thm:lambda} reveals that the behavior of the convergence rate $\lambda$ is essentially different in two ranges of momentum weights $\gamma$, exhibiting a phase transition at $\gamma=(1-\phi)^2/(1+\phi)^2$. 
With Theorem \ref{thm:lambda}, the following remark sheds light on the optimal choice of $\alpha$ to achieve the fastest convergence rate. 
\begin{remark}\label{rem4}
When $\gamma$ increases from $0$ to $1$, the spectral radius $\lambda$ first decreases and then increases, and the minimal spectral radius is achieved under the condition $\gamma=(1-\phi)^2/(1+\phi)^2$. Particularly, for $\gamma\leq (1-\phi)^2/(1+\phi)^2$, the quantity $\phi$ in Theorem \ref{thm:lambda} is non-decreasing, and 
$\lambda$ decreases as $\gamma$ increases. For $\gamma\geq (1-\phi)^2/(1+\phi)^2$, the spectral radius $\lambda=\sqrt{\gamma}$ increases as $\gamma$ increases. 
Moreover, the minimal spectral radius is 
\begin{eqnarray*}
    \lambda^* = \frac{\sqrt{L}-\sqrt{\mu}}{\sqrt{L}+\sqrt{\mu}},
\end{eqnarray*}
if we specify $\alpha=1/\sqrt{\mu L}$ and $\gamma = (\sqrt{L}-\sqrt{\mu})^2/(\sqrt{L}+\sqrt{\mu})^2$, and it follows that $\phi = \sqrt{\mu/L}$.
\end{remark}

Figure \ref{fig:lambda} illustrates the spectral radium $\lambda$ presented in Theorem \ref{thm:lambda} with respect to different $\gamma$ and $\alpha$ when $L/\mu$ is set to $5$. A brighter color corresponds to smaller $\lambda$ so that the convergence is faster. 
Figure \ref{fig:lambda} verifies that the fastest convergence rate is achieved when $\alpha$ is approximately $1/\sqrt{\mu L}$, $\gamma$ is approximately $(\sqrt 5-1)^2/(\sqrt 5+1)^2\approx0.146$ and the optimal $\lambda^*$ is near $\frac{\sqrt 5- \sqrt 1}{\sqrt 5+\sqrt 1}\approx 0.382$. 

\begin{remark}\label{rem3}
Figure \ref{fig:lambda} shows that SGDM with large batch sizes converges faster than SGD to a local neighborhood of $x^*$. This is reflected by the observation that, in the figure, the minimal radium $\lambda$ for SGDM (with best $\gamma\approx0.146$) is smaller than the minimal $\lambda$ of SGD ($\gamma=0$). A smaller $\lambda$ for SGDM corresponds to a smaller second term in \eqref{eq:last-it} of Theorem \ref{thm:g_gamma}, and thus implies that SGDM will converge faster to enter a local neighborhood of $x^*$ determined by the first term in \eqref{eq:last-it}. This indicates a faster convergence rate of SGDM in the initial stage with mini-batching. 

Moreover, compared to SGD ($\gamma=0$), SGDM with a larger $\gamma$ permits more flexible choices of the learning rate $\alpha$, i.e., the colored region in the figure is larger for larger $\gamma$. This means that the convergence of SGDM is less sensitive to learning rates. 
The conclusion is validated by numerical experiments in the subsequent section. Particularly, Figure \ref{fig:dif-alpha} in Section \ref{sec:exp_sensitivity} shows that SGDM permits a wider range of learning rates to achieve convergence.
\end{remark}
\begin{figure}[ht!]
    \centering
    \includegraphics[width=6.7cm]{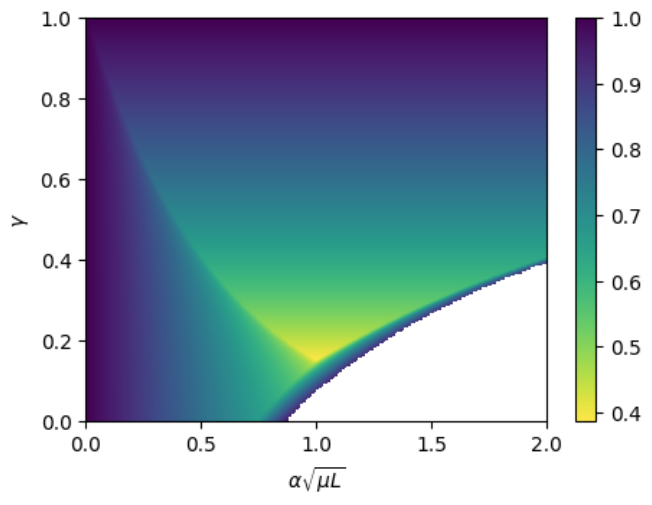}
    {\caption{The spectral radius of $\Gamma$ with respect to $\gamma$ and $\alpha$ in Theorem \ref{thm:lambda}, where $L/\mu=5/1$.}
    \label{fig:lambda}}
\end{figure}

\subsubsection{Explicit convergence rates for SGDM with specified $\gamma$ and $\alpha$}
In the following, we formalize the conclusions drawn from Theorem \ref{thm:g_gamma} and Remark \ref{rem4} in two corollaries on the explicit convergence rates of SGDM within the two phases of specifications of $\gamma$. 

\begin{corollary}[Small momentum weight $\gamma$] \label{cor:rate}
    Under the assumptions in Theorem \ref{thm:g_gamma}, for $0\leq \gamma\leq (1-\alpha L)/(4+4\alpha^2 L^2)$, we have that $\Delta\geq (1-\alpha L)^2/2$, $\alpha L^2 L_f^2 =\cO(\batch\mu)$ and 
    \begin{eqnarray*}
        \mE[\norm{\widetilde{m}_{t+1}}^2+\norm{{x}_{t+1}-x^*}^2]= \cO\bracket{\frac{L^2}{\batch\mu}\alpha\sigma^2 + L^2\lambda^{2(1-\delta)t}}.
    \end{eqnarray*}
\end{corollary}
Corollary \ref{cor:rate} establishes the explicit finite-sample rate of convergence of SGDM under small $\gamma$, and shows that, the spectral radius $\lambda$ decreases slowly as $\gamma$ increases from $0$ to $(1-\alpha L)/(4+4\alpha^2 L^2)$. 

\begin{corollary}[Large momentum weight $\gamma$] \label{cor:gamma}
    Under the assumptions in Theorem \ref{thm:g_gamma},  for $0<1-\gamma\leq 2\alpha\mu/(1+\alpha\mu)^2$, we have that $\Delta\geq 2 (1-\gamma)\alpha\mu$,  $\alpha  L_f^2\bracket{L^2+\frac{\alpha^2}{(1-\gamma)^2}} =\cO\bracket{\batch\mu},$ and
    \begin{eqnarray*}
        \mE[\norm{\widetilde{m}_{t+1}}^2+\norm{{x}_{t+1}-x^*}^2]= \cO\bracket{\frac{1}{\batch\mu}\bracket{L^2+\frac{\alpha^2}{(1-\gamma)^2}}\alpha\sigma^2+ \bracket{L^2+\frac{\alpha}{(1-\gamma)\mu}}\lambda^{2(1-\delta)t} },
    \end{eqnarray*}
    where $\lambda$ is defined in \eqref{eq:lambda} above. Furthermore, if $\alpha< \frac{2(1+\gamma)}{(\mu+L)(1-\gamma)}$, we have $\lambda=\sqrt{\gamma}$.
    
\end{corollary}
Corollary \ref{cor:gamma} demonstrates that the spectral radius $\lambda$ increases with $\gamma$ when it approaches $1$, exhibiting an opposite behavior compared to the small $\gamma$ settings that in Corollary \ref{cor:rate}.
As we demonstrated in Remark \ref{rem3} above, compared to mini-batch SGD ($\gamma=0$), mini-batch  SGDM with an appropriate $\gamma$ converges faster to the local neighborhood of $x^*$ by achieving a smaller second term in their convergence rates. The following Remark \ref{re:6} implements an optimal choice of the momentum weight $\gamma$ in SGDM, to provide explicit convergence rates of $\widetilde{m}_{t+1}$ and ${x}_{t+1}$ with respect to $L$, $\mu$ and $B$.
\begin{remark}\label{re:6}
    As one increases $\gamma$ from $0$ to $1$, the corresponding spectral radius $\lambda$ first decreases and then increases. 
    Particularly, if one specifies $\gamma$ and $\alpha$ as
\begin{eqnarray*}
    \gamma = \frac{(1-c\alpha \mu)^2}{(1+c\alpha \mu)^2}, \quad \text{and} \quad \alpha\leq  \sqrt{\frac{1}{\mu L}},
\end{eqnarray*}
for $c<1$ sufficiently close to 1, we have that $\Delta=\cO(\alpha^2\mu^2)$ and
\begin{eqnarray*}
    \mE[\norm{\widetilde{m}_{t+1}}^2+\norm{{x}_{t+1}-x^*}^2]= \cO\bracket{\frac{(L+1/\mu)^2}{ \batch\mu}\alpha\sigma^2+(L+1/\mu)^2\gamma^{(1-\delta)t} }.
\end{eqnarray*}
\end{remark}

We now compare our rate of convergence with several results in the existing literature.

\paragraph{Comparison to SGD.}
Theorem 4.6 in \cite{bottou2018optimization} proved the convergence rate of SGD under strong convexity $\cO(\frac{L}{\mu}\alpha \sigma^2+(1-\alpha \mu)^t)$ under the assumption $\mE[\norm{\nabla f_{\xi}(x)-\nabla f(x)}^2]\leq \sigma^2$ for all $x$. Comparing with their result, our analysis in Theorem \ref{thm:lambda} can apply with $\gamma=0$ and $\alpha< 2/(\mu+L)$ such that $\phi=\alpha\mu$, which is consistent with \cite{bottou2018optimization}. 
Our Remark \ref{rem4} also shows that SGDM with large momentum weight can accelerate the convergence rate in \cite{bottou2018optimization}. 

\paragraph{Comparison to existing results for SGDM.}
\cite{liu2020improved} proved that the convergence rate for strongly convex loss is $\cO(\frac{L}{\mu}\alpha\sigma^2+\max\{1-\alpha\mu, \gamma\}^t)$ for $t\geq t_0=\cO(-1/\log(\gamma))$. Comparing that with Corollary \ref{cor:rate} above, our result improves the rate when the momentum weight $\gamma$ is small. 
Moreover, \cite{liu2020improved} required learning rate $\alpha = \cO(1-\gamma)$, which is considerably more restrictive than the range permitted in our Theorems that $\alpha L<2(1+\gamma)/(1-\gamma)$. This requirement is in opposition to our finding that SGDM permits a wider range of learning rates. The stochastic heavy ball (SHB) method is equivalent to SGDM for $\alpha'=\alpha(1-\gamma)$ and $\gamma'=\gamma$. Under the noiseless setting, for $\alpha'=1/L$ and $\gamma'= (\sqrt{L/\mu}-1)/(\sqrt{L/\mu}+1)$, \cite{nesterov1983method} gave the $\cO((1-\sqrt{\mu/L})^t)$ convergence rate. For $\alpha'=\cO(1/L)$ and $\alpha=\cO(\sqrt{1/\mu L})$ in Remark \ref{re:6}, our convergence factor is $\gamma =\cO(1-\sqrt{\mu/L})$, which improves the factor $\cO(1-\mu/L)$ in SGD and is consistent with \cite{nesterov1983method}. Theorem 3.5 in \cite{bollapragada2022fast} proved that for $
    \sqrt{\alpha'} = \frac{2}{\sqrt{L+\epsilon}+\sqrt{\mu-\epsilon}}$ and $\sqrt{\gamma'} = \frac{\sqrt{(L+\epsilon)/(\mu-\epsilon)}-1}{\sqrt{(L+\epsilon)/(\mu-\epsilon)}+1}$,  $\epsilon\in (0,\mu)$, the convergence rate of the SHB method on a consistent linear system $Ax=b$ is $\cO(\frac{L^2}{\epsilon}t^2(\gamma')^t+\frac{L^4}{\epsilon\mu^2}\sigma^2)$. Their convergence rate is the same as ours in Remark \ref{re:6} under a different setting, while their results are more restrictive on the choice of learning rates and momentum weights. 

\subsection{The finite-sample rates for SGDM on general losses} For the general strongly convex loss where $\overline L>0$, we provide the following convergence rate.

\begin{theorem}\label{thm:lrate}
Under (A1), (A2), and (A3') and $\overline{L}>0$, for any momentum $\gamma\in[0,1)$, assume the learning rate $\alpha>0$ satisfies $\alpha L<2(1+\gamma)/(1-\gamma)$ and
\begin{eqnarray*}
    6\sqrt{2}c  (2\log T+4d)  M\bracket{\frac{2 L_f}{\lambda^{(1-\delta)}\sqrt{\delta}(1-\lambda)^{1/2}}+ \frac{3 \overline{L} M \alpha\sigma}{(1-\lambda)^{3/2}}}\alpha\leq \sqrt{\batch},
\end{eqnarray*}
for fixed $\delta\in(0,1]$ and an absolute constant $c>0$, where $M$ is defined in (\ref{def:m}). In addition, the initialization $x_1$ satisfies $ \norm{x_1-x^*}\leq \frac{\delta(1-\lambda)\lambda^{1-\delta}}{   36 M^2  
 \alpha \overline{L}} $. 
Then with probability $1-2 T^{-1}$, 
\begin{eqnarray*}
\sqrt{\norm{\widetilde{m}_{t+1}}^2+\norm{x_{t+1}-x^*}^2}\leq \frac{3\sqrt{2}c \bracket{2\log T+4d} M}{\sqrt{\batch(1-\lambda)}}\alpha \sigma  + 3 M \norm{x_1-x^*}\lambda^{(1-\delta)t},~ \text{for } 1\leq t\leq T.
\end{eqnarray*}
\end{theorem}


Compared to Theorem \ref{thm:g_gamma}, the convergence rate in Theorem \ref{thm:lrate} is worse but with only up to a logarithm term, which indicates that SGDM for general losses achieves almost the same rate as that for quadratic ones. 
This convergence bound is presented with high probability for general losses, in contrast to the convergence in expectation outlined for the quadratic losses previously. Nonetheless, it remains uniformly applicable across all iterations $T$.

Mini-batch SGDM for general convex loss still converges faster than mini-batch SGD with high probability, as discussed in Theorem \ref{thm:lambda}.
In contrast to the quadratic settings $\overline{L}=0$, the general convex setting requires additionally that the initialization $x_1$ is close to $x^*$, i.e., $\norm{x_1-x^*}$ is small. This assumption is expected since the loss function is not guaranteed to be convex everywhere under the weak assumptions in (A1)--(A2). Particularly, for $\gamma$ close to 0, we assume $\norm{x_1-x^*} = \cO\bracket{\mu/\overline{L}L^2}$. For $\gamma$ close to $1$, we assume $\norm{x_1-x^*} = \cO\bracket{\mu/\overline{L}\bracket{L^2+\alpha^2/(1-\gamma)^2}}$.

Based on Theorem \ref{thm:lrate}, we have the following corollaries of the convergence rates for small and large specifications of $\gamma$ under general losses, similar to Corollaries \ref{cor:rate}--\ref{cor:gamma} for quadratic losses. 

\begin{corollary}[Small momentum weight $\gamma$]
    Following Theorem \ref{thm:lrate}, for $0\leq \gamma\leq (1-\alpha L)/(4+4\alpha^2 L^2)$, we have that $\Delta\geq (1-\alpha L)^2/2$, $\sqrt{\alpha} L_f L\bracket{\mu+\overline{L}L} = \cO\bracket{\batch^{1/2} \mu^{3/2}}$ and with high probability,
    \begin{eqnarray*}
    \sqrt{\norm{\widetilde{m}_{t+1}}^2+\norm{x_{t+1}-x^*}^2}=\cO\bracket{\frac{L\log T}{\sqrt{\batch\mu}}\sqrt{\alpha}\sigma+L\lambda^{(1-\delta)t}},
    \end{eqnarray*}    where $\lambda$ is defined in \eqref{eq:lambda} above. 
\end{corollary}
\begin{corollary}[Large momentum weight $\gamma$]
    Following Theorem \ref{thm:lrate}, for $0<1-\gamma\leq 2\alpha\mu/(1+\alpha\mu)^2$, we have that $\Delta\geq 2 (1-\gamma)\alpha\mu$, $$\sqrt{\alpha}  L_f\bracket{L+\frac{\alpha}{1-\gamma}} \bbracket{1 + \overline{L} \bracket{L+\frac{\alpha}{1-\gamma}} \sqrt{\frac{\alpha}{(1-\gamma)\mu}}} = \cO\bracket{\batch\mu },$$
    and with high probability
    \begin{eqnarray*}
        \sqrt{\norm{\widetilde{m}_{t+1}}^2+\norm{x_{t+1}-x^*}^2}=\cO\bracket{\frac{\log T}{\sqrt{\batch\mu}}\bracket{L+\frac{\alpha}{1-\gamma}}\sqrt{\alpha}\sigma+ \bracket{L+\frac{\sqrt{\alpha}}{\sqrt{(1-\gamma)\mu}}}\lambda^{(1-\delta)t}},
    \end{eqnarray*}
    where $\lambda$ is defined in \eqref{eq:lambda} above. Furthermore, if $\alpha< \frac{2(1+\gamma)}{(\mu+L)(1-\gamma)}$, we have $\lambda=\sqrt{\gamma}$.
\end{corollary}
In the following remark, we choose the optimal momentum weights $\gamma$ to show the explicit convergence results with respect to $L$, $\mu$, and $\batch$. 
\begin{remark}
If one specifies $\gamma$ and $\alpha$ as
\begin{eqnarray*}
    \gamma = \frac{(1-c\alpha \mu)^2}{(1+c\alpha \mu)^2}, \quad \text{and} \quad \alpha\leq  \sqrt{\frac{1}{\mu L}},
\end{eqnarray*}
for $c<1$ sufficiently close to 1, we have that $\Delta=\cO(\alpha^2\mu^2)$ and with high probability
\begin{eqnarray*}
    \norm{\widetilde{m}_{t+1}}^2+\norm{x_{t+1}-x^*}^2= \cO\bracket{\frac{(L+1/\mu)^2(\log T)^2}{ \batch\mu}\alpha\sigma^2+(L+1/\mu)^2 \gamma^{(1-\delta)t} }.
\end{eqnarray*}
We arrive at the same conclusions as we had in quadratic settings that SGDM can accelerate convergence over SGD under general strongly convex losses when $\bar L>0$.
\end{remark}

\section{Acceleration by Averaging and Asymptotic Normality}\label{sec:average}
In this section, we study the Polyak-averaging SGDM (referred to as \emph{averaged SGDM}). Particularly, we aim at building the convergence result of $\frac{1}{n-n_0}\sum_{t=n_0+1}^n  x_t$, which is an average of all iterates $x_t$ starting from a constant period $n_0$. We show that the averaging leads to acceleration for SGDM, compared to the last-iterate bounds established in the previous section. 

\subsection{Averaged SGDM under quadratic losses}
For the mini-batch model (\ref{so2}) under quadratic losses  $(\overline{L}=0)$, 
we now characterize a decomposition of the error of the averaged SGDM $\frac{1}{n-n_0}\sum_{t=n_0+1}^n (x_t-x^*)$, which leads to an acceleration over the last iterate of SGDM $x_n$. 

\begin{theorem}\label{th2} Under (A1)-(A3) and $\overline{L}=0$, suppose that the conditions in Theorem \ref{thm:g_gamma} hold, and the $t$-th iteration $(\widetilde m_t,x_t)$ satisfies
\begin{eqnarray}\label{eq:repeat-last}
\mE[\norm{\widetilde{m}_{t}}^2+\norm{x_{t}-x^*}^2]\leq \frac{C_1}{\batch(1-\lambda)} \alpha^2 \sigma^2 + C_2 \norm{x_1-x^*}^2 \lambda^{2(1-\delta)(t-1)}, \quad t\geq 1.
\end{eqnarray}
Then for $n\geq 2n_0$ where $\lambda^{2n_0}=\batch^{-1}(1-\lambda)$, we have
\begin{eqnarray}\label{eq:th2-main}
    \frac{\sum_{t=n_0+1}^n (x_t-x^*)}{n-n_0} = -  \frac{\sum_{t=n_0+1}^n \sum_{i=1}^{\batch}\Sigma^{-1}(A_{\xi_{ti}}x^{*}-b_{\xi_{ti}})}{\batch(n-n_0)} + R_n,
\end{eqnarray}
where $\mE\norm{R_n}^2\leq \frac{\tilde{C}_1}{\batch n^2} +  \frac{\tilde{C}_2}{\batch^2 n}$, and 
\begin{eqnarray*}
&& \tilde{C}_1 = 40 \alpha^2 L_f^2 M^2 \frac{C_2\norm{x_1-x^*}^2}{(1-\delta)(1-\lambda)^3} + 120 \alpha^2\sigma^2  M^2 \frac{1}{(1-\lambda)^3}+   \frac{4 M^2 \norm{x_1-x^*}^2}{1-\lambda},\\
&& \tilde{C}_2 = 30 \alpha^2\sigma^2 L_f^2  M^2\frac{C_1\alpha^2}{(1-\lambda)^3}.
\end{eqnarray*}
and $M$ is defined in (\ref{def:m}). 
\end{theorem}
In the above theorem, we obtain the convergence rate of averaged SGDM in \eqref{eq:th2-main}. While \eqref{eq:repeat-last} is assumed for the purpose of explicitly establishing the dependence of $\tilde C_1$, $\tilde C_2$ to the constants $C_1$, $C_2$ in \eqref{eq:repeat-last}, it is nothing but a rephrasing of the last-iterate bounds established in \eqref{eq:last-it}. 
Comparing the convergence rates in Theorem \ref{th2}, the averaged SGDM converges to $x^*$ without a bias as $n$ increases. The leading term $-\frac{\sum_{t=n_0+1}^n \Sigma^{-1} (A_{\xi_{ti}}x^{*}-b_{\xi_{ti}})}{n-n_0} $ is an average of $(n-n_0)$ i.i.d. vectors when $\xi_{ti}$ are i.i.d. sampled.  Therefore the leading term is bounded by $\cO(\frac{1}{\batch n})$ in squared expectation, and the remainder term $R_n$ in \eqref{eq:th2-main} is bounded by $\frac{\tilde{C}_1}{\batch n^2} +  \frac{\tilde{C}_2}{\batch^2 n}$. We demonstrate that the application of the averaging technique enhances the convergence rate of SGDM from the biased expression presented in (\ref{eq:repeat-last}) to an unbiased rate of $\mathcal{O}(\frac{1}{n})$, which asymptotically approaches zero as $n\rightarrow \infty$. It is noteworthy to mention that, for averaged SGDM, the initialization error $\|x_1-x^*\|$ is forgotten at the rate of $n^{-2}$, slower than the exponential initialization-forgetting for the last-iterate convergence established in Theorem \ref{thm:g_gamma}. 

\begin{remark}\label{re:n0}
    Comparing the convergence rates of SGDM and averaged SGDM in Theorem \ref{thm:g_gamma} and Theorem \ref{th2}, respectively, we establish that the leading term in the averaged SGDM does not depend on the learning rate $\alpha$. On the other hand, the averaging process starts from $n_0$, effectively excluding the initial iterations that may have larger deviations. Practically, the starting point $n_0$ is manually selected due to $\mu$ and $L$. With a small $\lambda$, averaged SGDM affords the selection of a smaller $n_0$ to fulfill the condition $\lambda^{2n_0} = \batch^{-1}(1-\lambda)$, thereby demonstrating reduced sensitivity to  $n_0$. In addition, the leading term in \eqref{eq:th2-main} is independent of momentum weight $\gamma$ and learning rate $\alpha$, which indeed indicates that the averaged SGDM has the same rate of convergence as the averaged SGD. 
\end{remark}

The following Corollary \ref{clt} establishes the asymptotic distribution of averaged SGDM, which is indeed the distribution of the leading term in \eqref{eq:th2-main}, and the remainder term $R_n$ in \eqref{eq:th2-main} establishes the convergence of the averaged SGDM algorithm to the asymptotic normal distribution.
\begin{corollary}\label{clt} Following Theorem \ref{th2}, we have
\begin{eqnarray*}
\frac{\sqrt{\batch}}{\sigma}
\frac{\sum_{t=n_{0}+1}^{n}({x}_{t}-x^*)}{\sqrt{n-n_{0}}}
\Rightarrow \mathcal{N}(0,\Sigma^{-1}\Omega\Sigma^{-1}), \quad \text{as}~ \bracket{\alpha n,\frac{\batch}{\alpha}}\to \infty,
\end{eqnarray*}
where
\begin{eqnarray}\label{omega}
\Omega = \frac{1}{\sigma^2}\mE[(A_{\xi}x^{*}-b_{\xi}) (A_{\xi}x^{*}-b_{\xi})^\top].
\end{eqnarray}
\end{corollary}

The asymptotic distribution of the averaged SGDM is the same as that of the averaged SGD in the existing literature \citep{polyak1992acceleration}, with different remainder terms $R_n$. 
As the averaged SGDM is close to $x^*$, the limiting distribution is established with covariance matrix $\Sigma^{-1}\Omega\Sigma^{-1}$ where $\Sigma$ and $\Omega$ are respectively the Hessian and Gram matrix of $A_{\xi}x^*-b_{\xi}$. 

\begin{remark}
Asymptotic normality in Corollary \ref{clt} hold under asymptotics $\alpha n, \batch/\alpha\rightarrow\infty$. We specify three common scenarios of the learning rate $\alpha$, batch size $\batch$, and sample size $n$. 
\begin{itemize}
\item For a constant learning rate $\alpha$, Corollary \ref{clt} holds as $n$ and $\batch$ tend to infinity. As the batch size $\batch$ increases, the bias term in (\ref{eq:repeat-last}) approaches zero, reducing the difference between the stochastic gradient at $x_t$ and $x^*$. 
\item 
For decaying learning rate $\alpha$, the asymptotic normality holds with a fixed batch size $\batch$ as long as $\alpha n$ diverges. The idea is intuitive: as $\alpha$ decreases, the random effect reduces. For instance, Corollary \ref{clt} holds  when the batch size $\batch$ is fixed and $\alpha=\Theta(n^{-\epsilon})$ with $\epsilon\in(0,1)$ as $n\to \infty$. 
\item By minimizing the remainder term $R_n$ in \eqref{eq:repeat-last}, the optimal learning rate is $\alpha=\Theta((n/\batch)^{-1/2})$. With such a specified learning rate, the averaged SGDM correspondingly converges to asymptotic normality with rate $\cO((n\batch)^{-1/2})$.
\end{itemize}
\end{remark}

In addition, Corollary \ref{clt} provides a rigorous foundation for constructing asymptotically valid confidence intervals based on the asymptotic normality of averaged SGDM. By estimating the covariance matrix, uncertainty quantification and statistical inference can be performed based on SGDM methods. Based on the asymptotic normality and the given covariance matrix, the $95\%$ confidence region of $x^*$ can be constructed as 
\begin{eqnarray*}
    I_n = \bbracket{x\in\mR^d: \frac{\batch(n-n_0)}{\sigma^2}\bracket{\frac{\sum_{t=n_{0}+1}^{n} x_t}{n-n_0}-x}^\top \Sigma \Omega^{-1} \Sigma\bracket{\frac{\sum_{t=n_{0}+1}^{n} x_t}{n-n_0}-x}\leq \chi^2_{d,0.05}},
\end{eqnarray*}
where $\chi^2_{d,0.05}$ is the $0.95$-quantile of the chi-squared distribution with $d$ degrees of freedom. Notably, for $\omega\in \mR^d$ with $\norm{\omega}=1$ as a test vector, we can define the random variable
\begin{eqnarray}\label{def:z}
    Z =  \frac{\sqrt{\batch}}{\sigma\sqrt{\omega^\top \Sigma^{-1}\Omega\Sigma^{-1} \omega }}\frac{\sum_{t=n_{0}+1}^{n}\omega^\top (x_t-x^*)}{\sqrt{n-n_{0}}}\Rightarrow \mathcal{N}(0,1), \quad \text{as} \quad n\to \infty,
\end{eqnarray}
based on which, one can construct a one-dimensional asymptotic exact confidence interval for $\omega^\top x^*$, 
\[
I_n^{\omega}=\left[\frac{\omega^\top\sum_{t=n_{0}+1}^{n} x_t}{n-n_0}-z_{0.025}\frac{\sigma\sqrt{\omega^\top \Sigma^{-1}\Omega\Sigma^{-1} \omega }}{\sqrt{\batch(n-n_0)}},\frac{\omega^\top\sum_{t=n_{0}+1}^{n} x_t}{n-n_0}+z_{0.025}\frac{\sigma\sqrt{\omega^\top \Sigma^{-1}\Omega\Sigma^{-1} \omega }}{\sqrt{\batch(n-n_0)}}\right],
\]
where $z_{0.025}$ is the $0.975$-quantile of the standard normal distribution. 
Particularly, $\Pr(\omega^\top x^*\in I_n^{\omega})\rightarrow 0.95$, as $n\rightarrow \infty$. 
In Section \ref{sec:exp_sensitivity} below, we conduct a simulation experiment on constructing the confidence intervals and report the outcomes in Figure \ref{fig:clt-alpha}. The coverage of the proposed construction is persuasive for averaged SGDM, and its performance benefits from the less sensitivity to the learning rates. The confidence interval can also serve as an uncertainty quantification of the averaged SGDM estimator $\frac{\omega^\top\sum_{t=n_{0}+1}^{n} x_t}{n-n_0}$, and one may consider the length of the confidence interval as a criterion for stopping the algorithm or other decision-making purposes. 

\subsection{Averaged SGDM under general losses}
For the general strongly convex loss function, we provide the convergence rate of the averaged SGDM in the following theorem.

\begin{theorem}\label{th3} For $\overline{L}>0$, suppose that the conditions in Theorem \ref{thm:lrate} hold, and the $t$-th iteration $(\widetilde m_t,x_t)$ satisfies for fixed $\delta\in(0,\frac{1}{2}]$,
\begin{eqnarray}\label{eq:th3-last}
    \sqrt{\norm{\widetilde{m}_{t}}^2+\norm{{x}_{t}-x^*}^2}\leq \frac{C_1}{\sqrt{\batch(1-\lambda)}}  \alpha \sigma + C_2 \norm{x_1-x^*} \lambda^{(1-\delta)(t-1)}, \quad 1\leq t\leq T.
\end{eqnarray}
Then for $2n_0\leq n\leq T$, where $\lambda^{n_0}=\batch^{-1/2} (1-\lambda)^{1/2}$, with probability of at least $1-4T^{-1}$, we have 
\begin{eqnarray}\label{eq:th3}
    \frac{\sum_{t=n_0+1}^n (x_t-x^*)}{n-n_0} = -\frac{\sum_{t=n_0+1}^n \sum_{i=1}^\batch \Sigma^{-1}\triangledown f_{\xi_{ti}}(x^*)}{\batch(n-n_0)} + R_n,
\end{eqnarray}
where 
\begin{eqnarray}\label{eqrn:l>0}
    \norm{R_n}\leq \frac{\Tilde{C}_1+\Tilde{C}_3}{\sqrt{\batch}n} + \frac{\Tilde{C}_2}{\batch\sqrt{n}}  + \frac{\Tilde{C}_4}{\batch},
\end{eqnarray}
\begin{eqnarray*}
    &&\Tilde{C}_1= \alpha L_f M  \frac{16c(2\log T+ 4d) C_2 \norm{x_1-x^*}}{(1-\lambda)^{3/2}}+\alpha \sigma M \frac{4\sqrt{3}c(2\log T+ 4d) }{(1-\lambda)^{3/2}}+  \frac{2M \norm{x_1-x^*}}{(1-\lambda)^{1/2}},\\
    &&\Tilde{C}_2 = \alpha L_f \sigma M  \frac{8\sqrt{2}c(2\log T+4d)C_1 \alpha}{(1-\lambda)^{3/2}},~\Tilde{C}_3 = 8\alpha M \overline{L} \frac{ C_2^2\norm{x_1-x^*}^2\bracket{n_0(1-\lambda)+1}}{(1-\lambda)^{3/2}} ,~ \Tilde{C}_4 = 4\alpha \sigma^2 M \overline{L} \frac{C_1^2 \alpha^2}{(1-\lambda)^2}.
\end{eqnarray*}
Here $M$ is defined in \eqref{def:m}.
\end{theorem}


For non-quadratic settings of $\overline{L}>0$, there are two additional terms $\tilde C_3/(\sqrt{\batch}n)$ and $\tilde C_4/\batch$ in \eqref{eqrn:l>0} compared to the case of $\overline{L}=0$. As $n\rightarrow\infty$, the term $\tilde C_4/\batch$ appears as a non-vanishing bias, if the batch size $\batch$ and learning rate $\alpha$ both stays fixed, due to which the asymptotic normality does not hold. Nonetheless, if one decays $\alpha$ or increases $\batch$ appropriately as $n$ increases, the asymptotic normality result remains to hold in the following corollary. 
 
\begin{corollary}\label{clt2} Following Theorem \ref{th3}, we have
\begin{eqnarray*}
\frac{\sqrt{\batch}}{\sigma}
\frac{\sum_{t=n_{0}+1}^{n}({x}_{t}-x^*)}{\sqrt{n-n_{0}}}
\Rightarrow \mathcal{N}(0,\Sigma^{-1}\Omega\Sigma^{-1}), \quad \text{as}~\bracket{\frac{\sqrt{\alpha n}}{\log (\batch/\alpha)}, \frac{\sqrt{\batch}}{\alpha\sqrt{n}}}\to \infty,
\end{eqnarray*}
where 
\begin{eqnarray}
    \Omega = \frac{1}{\sigma^2} \mE[\triangledown f_\xi(x^*) \triangledown f_\xi(x^*)^\top].
\end{eqnarray}
\end{corollary}


\begin{remark}The conditions on $n$, $\batch$, and $\alpha$ in this corollary are more restrictive than those in Corollary \ref{clt} due to the presence of the approximation errors $\overline{L}\norm{x-x^*}^2$. For a decaying learning rate $\alpha$, the condition on the batch size is much relaxed since a small learning rate reduces the error caused by randomness, as shown in the $C_1$ term in \eqref{eq:th3-last}. Particularly, when the batch size $\batch$ is fixed, the condition of Corollary \ref{clt2} is met for $\alpha=\Theta(n^{-\epsilon})$ with $\epsilon\in(\frac{1}{2},1)$ as $n\to \infty$. Specifically, with diverging $n$ and $\batch$, the nearly optimal learning rate is $\alpha=\Theta((n/\sqrt{\batch})^{-2/3})$, and the corresponding rate of convergence to asymptotic normality is $\cO((n\batch)^{-1/3}\log^2(n\batch))$.
\end{remark}

\section{Experiments}\label{sec:exp}
In this section, we support our theoretical results with simulations in a quadratic example and a logistic loss example of the general strongly convex loss, and real-data experiments of a multinomial logistic regression on MNIST hand-written digit classification. 

In the numerical experiments, we verify the convergence results built in the paper on a training set of size $N$. Setting $\xi$ to be a discrete uniform random variable with values in $\{1,2,\dots, N\}$, we consider to minimize $\min_x \frac{1}{N}\sum_{i=1}^N f_i(x)$, as an example of the stochastic optimization model (\ref{so}), 
where the objective is either the empirical risk $\mE [ f_{\xi}(x)]=\frac{1}{N}\sum_{i=1}^N f_i(x)$, or more specifically in a  maximum likelihood estimation, the negative log-likelihood. In all the simulation results below, we repeat the experiment 200 times. For the real-data analysis on MNIST, we repeat the experiment in a replicable setting while we fix the random seeds to 1, 2, and 3 in training. 
\subsection{Simulation: quadratic loss}
In a quadratic loss model (\ref{eq:quadratic}), we first conduct a simulation study with a sample of size $N=20,000$ and dimension $d=10$. We generate the parameters $\{(\bba_i,\bfb_i)\}_{i=1}^N$ in (\ref{eq:quadratic}) i.i.d and $\bfb_i\sim\mathcal{N}(0_d,I_{d})$, and the positive definitive matrix $\bba_i=\rho \bbv_i^\top \bbv_i + 10 I_d$. Here $\bbv_i$ is a matrix in $\mR^{d\times d}$ and each row of $\bbv_i$ generates from the normal distribution $\mathcal{N}(0_d,I_{d})$, where $\rho>0$ affect the conditional number $L/\mu$. We choose $\rho=1$ and the average conditional number $L/\mu$ is $35/10$. The deterministic minimizer $x^*$ is computed by $x^*=\bracket{\sum_{i=1}^N \bba_i}^{-1} \bracket{\sum_{i=1}^N \bfb_i}$.

We consider the mini-batch SGDM with batch size $\batch=0.2N$ with replacement. The learning rate is fixed at $\alpha = 0.001$. According to Theorem \ref{thm:lambda}, we choose the momentum weight following
\begin{eqnarray}\label{gammaset}
    \gamma = \bracket{\frac{1-\mu \alpha}{1+\mu \alpha}}^2.
\end{eqnarray}
We refer to the above $\gamma$ as the adaptive momentum weight in SGDM (\texttt{SGDM-adap}), while we also compare it with other fixed $\gamma$, as well as SGD as a special case of SGDM with $\gamma=0$ in Figure \ref{fig:linear}.

\begin{figure}[ht]
  \centering
  \subfigure[Small $\gamma$]{
    \includegraphics[width=6.7cm]{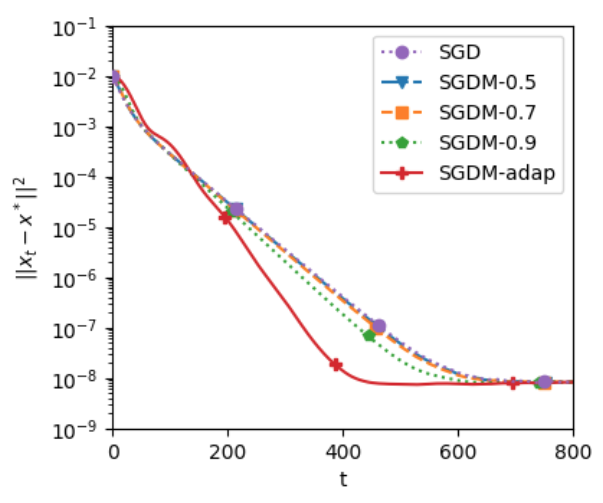}}
  \hspace{0in} 
  \subfigure[Large $\gamma$]{
    \includegraphics[width=6.7cm]{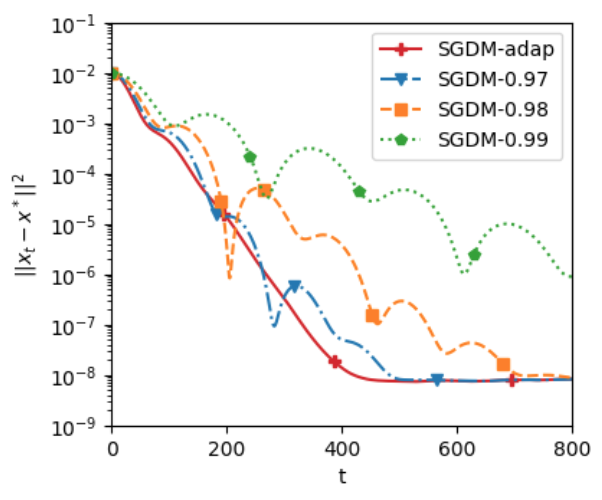}}
  \caption{Performance of SGD and SGDM on the quadratic loss. The average value of the adaptive momentum weight is 0.96.}
  \label{fig:linear}
\end{figure}


Figure \ref{fig:linear} supports the theoretical results in Theorems \ref{thm:g_gamma} and \ref{thm:lambda}. Both SGD and SGDM enjoy linear convergence at the beginning and have the same order of error at the end. The adaptive momentum \texttt{SGDM-adap} with $\gamma = (1-\phi)^2/(1+\phi)^2$ in (\ref{gammaset}) converges fastest, where the average value over 200 experiments is $0.95$. The acceleration of convergence with small momentum weights is not significant, while SGDM with a very large $\gamma$ becomes slower and unstable. Particularly, SGDM with $\gamma=0.5$ converges as fast as SGD, and SGDM with momentum weight 0.9 converges faster than SGD. For large momentum weight $\gamma=0.99$, the linear convergence factor is $\sqrt{\gamma}$, and we can see from the experiment that its convergence is the slowest in Figure \ref{fig:linear}.


We can see that appropriate momentum weights lead to faster convergence as shown in Figure  \ref{fig:linear}(a). However, when the momentum weight is exceedingly large, SGDM may cause the $\ell_2$ error to oscillate, as shown in Figure \ref{fig:linear}(b). This is because momentum causes the algorithm to continue moving in the direction of past gradients, even if the current gradient is opposite to the momentum. As a result, SGDM oscillates back and forth near the minimum of the loss function instead of converging steadily towards it. Moreover, it also leads to a decrease in the convergence speed. The observation matches the finding in Theorem \ref{thm:lambda}.




We further compare the averaged SGD and SGDM. We set $n_0$ in Theorem \ref{th2} as $n_0=200$ and $500$, and report the convergence of the averaged SGD and SGDM in
\begin{eqnarray}\label{def:ave}
    \bar{x}_{n_0+t} = \frac{\sum_{j=n_0+1}^{n_0+t} x_j}{t}.
\end{eqnarray}
\begin{figure}[ht]
  \centering
  \subfigure[Small $\gamma$]{
    \includegraphics[width=6.7cm]{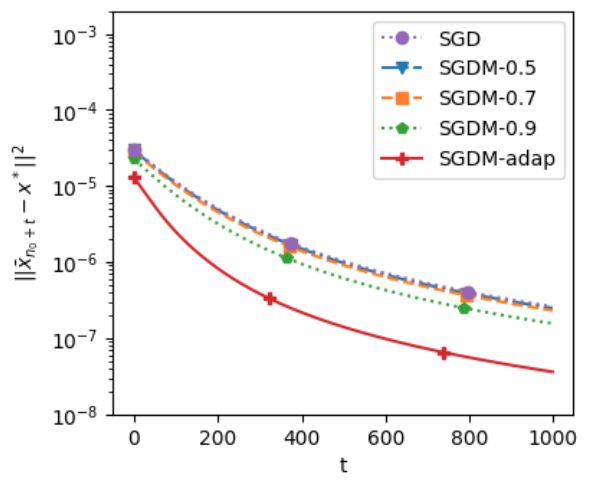}}
  \hspace{0in} 
  \subfigure[Large $\gamma$]{
    \includegraphics[width=6.7cm]{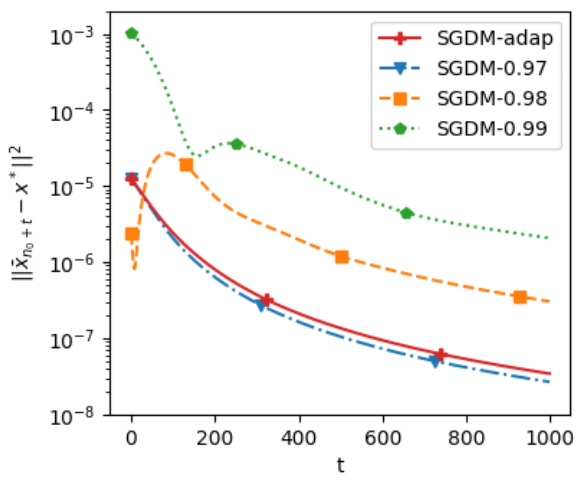}}
  \caption{Performance of Averaged SGD and Averaged SGDM on the quadratic loss. The average value of the adaptive momentum weight is 0.96.}
    \label{fig:average}
\end{figure}
\begin{figure}[ht]
  \centering
  \subfigure[Small $\gamma$]{
    \includegraphics[width=6.7cm]{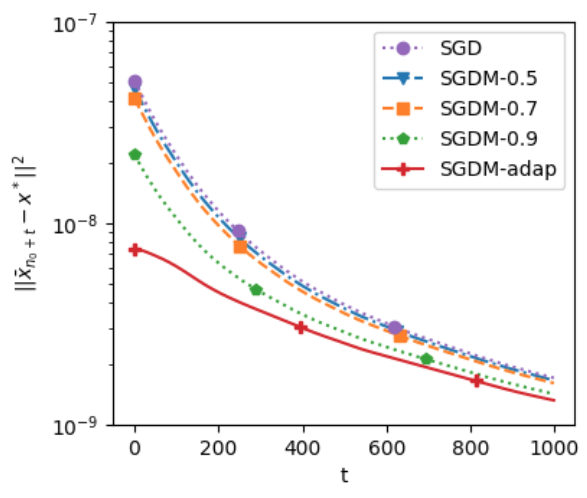}}
  \hspace{0in} 
  \subfigure[Large $\gamma$]{
    \includegraphics[width=6.7cm]{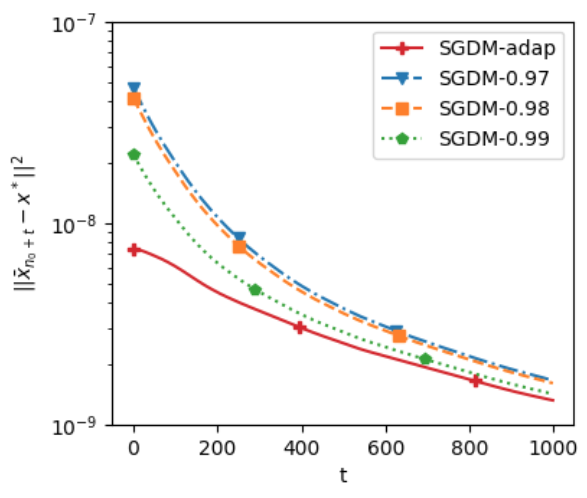}}
  \caption{Performance of Averaged SGD and Averaged SGDM on the quadratic loss with $n_0=500$.}
    \label{fig:average2}
\end{figure}
In Figures \ref{fig:average}--\ref{fig:average2}, averaged SGDM with adaptive momentum weight defined as (\ref{gammaset}) converges faster than the others. Specifically, it converges significantly faster than the averaged SGD. When $n_0=500$, Figure \ref{fig:average2} illustrates that the averaging technique reduces the order of the bias from $10^{-8}$, as seen in Figure \ref{fig:linear}, to $10^{-9}$. 

We further conduct a simulation study to verify the asymptotic normality result of Corollary \ref{clt}. In this experiment, we evaluate the one-dimensional projection statistic $Z$ in \eqref{def:z}
with $n_0=1000$, $n=2000$ and $\Omega=\bracket{N\sigma^2}^{-1}\sum_{i=1}^N (A_i x^*-b_i)(A_i x^*-b_i)^\top$. 
We generate the trajectory $\{x_1,\cdots,x_n\}$ 1000 times and get the replications $Z^{(1)}, Z^{(2)}, \cdots, Z^{(1000)}$. Figure \ref{fig:clt} shows the frequency of $\{Z^{(1)}, Z^{(2)}, \cdots, Z^{(1000)}\}$. From Figure \ref{fig:clt} we can see that both averaged SGD and SGDM well approximate the normal distribution. We can further see that the frequency of averaged SGDM is slightly closer to the normal distribution than that of averaged SGD under finite rounds $n=2000$. These observations reflect the theoretical results we build in Theorem \ref{th2}.


\begin{figure}[ht]
  \centering
  \subfigure{
    \includegraphics[width=6.2cm]{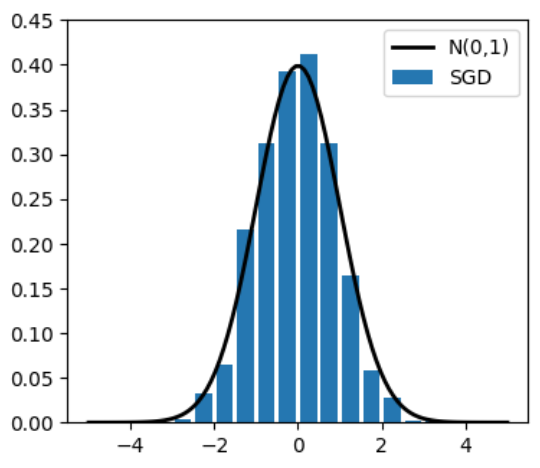}}
  \hspace{0in} 
  \subfigure{
    \includegraphics[width=6.2cm]{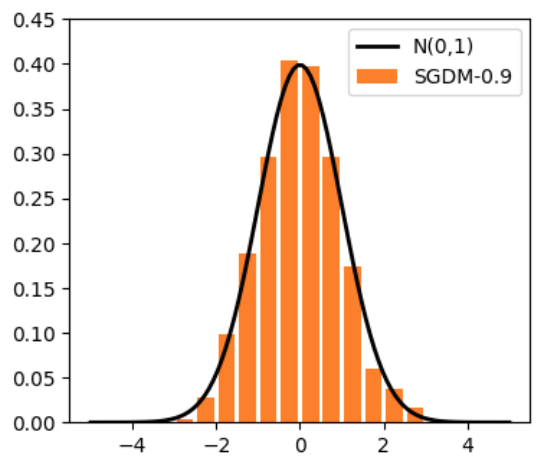}}
  \caption{Frequency of $Y$ about Averaged SGD and Averaged SGDM  with $\gamma=0.9$.} 
  \label{fig:clt}
\end{figure}

\subsection{Sensitivity to learning rates}\label{sec:exp_sensitivity}
In this section, we will show that the performance of SGDM and averaged SGDM have a wider range of tunable learning rates compared to SGD and averaged SGD.
In the quadratic loss model with $N=20,000$ and $d=10$, we generate the parameters $\{(\bba_i,\bfb_i)\}_{i=1}^N$ in (\ref{eq:quadratic}) i.i.d and $\bfb_i\sim\mathcal{N}(0_d,I_{d})$, and the positive definitive matrix $\bba_i= \bbv_i^\top \bbv_i + I_d$, where each row of $\bbv_i$ generates from the normal distribution $\mathcal{N}(0_d,I_{d})$. The average conditional number $L/\mu$ is $26/1$. The learning rates we set are chosen from $\{2^1,2^0,2^{-1},2^{-2},\cdots\}$. The batch size is $\batch=0.2N$. 

\begin{figure}[ht!]
    \centering
    \includegraphics[height=6cm]{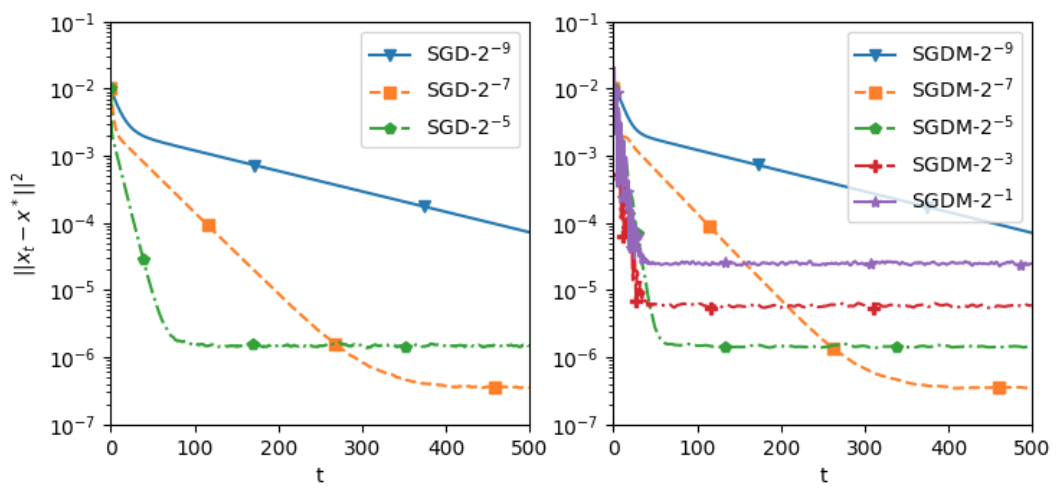}
    {\caption{Performance of SGD and SGDM with $\gamma=0.8$ under different learning rates. SGD fails to converge with $\alpha=2^{-1}$ and $\alpha=2^{-3}$.}
    \label{fig:dif-alpha}}
\end{figure}

Figure \ref{fig:dif-alpha} shows the convergence behaviors of SGD and SGDM for a wide range of learning rates $\alpha=2^{-1},2^{-3},2^{-5},2^{-7},2^{-9}$. Notably, SGD fails to converge with learning rates of $\alpha\geq 2^{-3}$. In contrast, SGDM with $\gamma=0.8$ exhibits robust convergence properties, successfully converging even with a larger learning rate $\alpha=2^{-1}$. This comparison emphasizes that SGDM is less sensitive to the learning rate. 
Figure \ref{fig:sen-alpha} further illustrates the finite-sample errors of SGD and SGDM across different learning rates, given a fixed number of iterations at $T=500$. 
Intuitively, smaller learning rates necessitate more iterations to converge and are associated with an increased bias for each method. Notably, SGDM is capable of converging with a relatively large learning rate. The maximum learning rate ensuring convergence is $\alpha=2^{-4}$ for SGD, $\alpha=2^{-1}$ for SGDM with $\gamma=0.8$ and $\alpha=2^{0}$ for SGDM with $\gamma=0.9$.
Reflecting upon the learning rate condition $\alpha L<(1+\gamma)/(1-\gamma)$, it can be confirmed that $ (1+\gamma)/(1-\gamma)\approx 2^{-1}/2^{-4}$ for $\gamma=0.8$ and $(1+\gamma)/(1-\gamma)\approx 2^{0}/2^{-4}$ for $\gamma=0.9$. This alignment validates that the figure is in agreement with the theoretical analysis.

\begin{figure}[ht!]
    \centering
    \includegraphics[width=6.7cm]{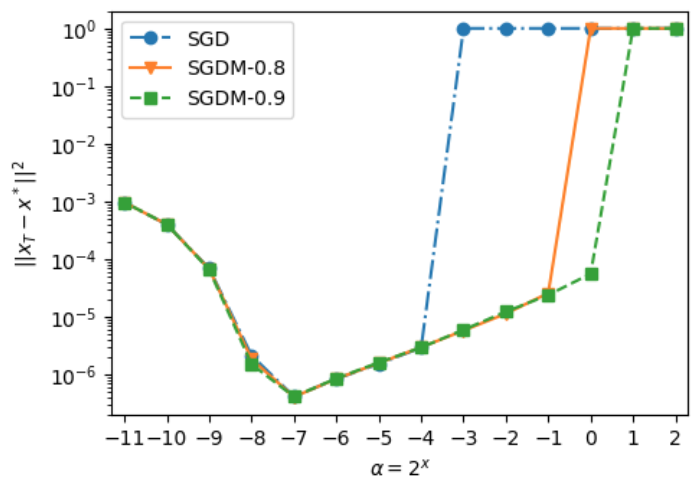}
    {\caption{Performance of SGD and SGDM with $\gamma=0.8,~0.9$ under different learning rates, where $T=500$ is fixed. }
    \label{fig:sen-alpha}}
\end{figure}

\begin{figure}[ht!]
 \centering
    \includegraphics[height=6cm]{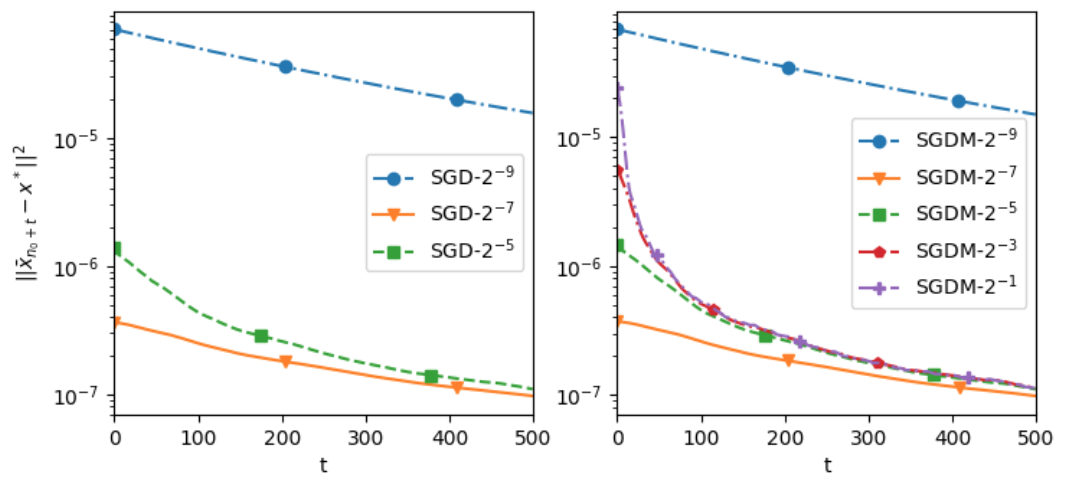}
    {\caption{Performance of Averaged SGD and Averaged SGDM with $\gamma=0.8$ under different learning rates.}
    \label{fig:ave-alpha}}
\end{figure}

Moreover, we examine the performance of averaged SGD and averaged SGDM across a range of learning rates. For both algorithms, we set $n_0=500$ and define $\bar{x}_{n_0+t}$ as (\ref{def:ave}). Figure \ref{fig:ave-alpha} illustrates that all algorithms exhibit a sublinear rate of convergence. SGDM with learning rates ranging from $\alpha=2^{-1}$ to $2^{-7}$, as well as SGD with learning rates $\alpha=2^{-5},2^{-7}$, achieve convergence to an similar error level. Compared to Figure \ref{fig:dif-alpha}, we can see that the averaging technique improves convergence and is less dependent on learning rates.
However, it is noteworthy that algorithms with learning rate $\alpha=2^{-9}$ require a substantially larger initial iteration count $n_0$ to mitigate the effects of iterations with significant deviations.

\begin{figure}[ht!]
    \centering
    \includegraphics[width=6.7cm]{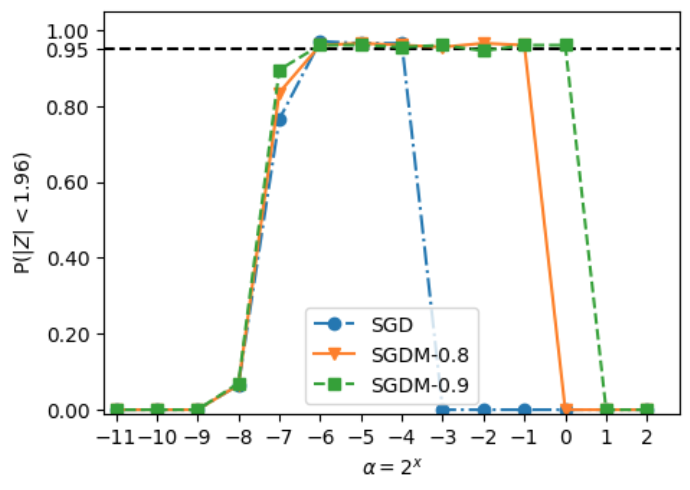}
    {\caption{The probability $P(|Z|<1.96)$ for SGD and SGDM with $\gamma=0.8,~0.9$ under different learning rates, where $Z$ is defined as (\ref{def:z}).}
    \label{fig:clt-alpha}}
\end{figure}

Additionally, we investigate the asymptotic normality as outlined in Corollary \ref{clt}. 
Figure \ref{fig:clt-alpha} illustrates the asymptotic behavior of the statistic $Z$ in (\ref{def:z}) with $n=1000$ and $n_0=500$.
This illustration confirms its convergence to $\mathcal{N}(0,1)$, under a range of learning rates. The $y$-axis displays the empirical probability $P(|Z|<1.96)$, which indicates the frequency with which the statistic $Z$ falls within the critical range $(-1.96, 1.96)$ across $1000$ trials, corresponding to a $95\%$ confidence interval. The $x$-axis specifies the learning rates employed in the numerical experiments.
Figure \ref{fig:clt-alpha} supports the assertions in the corollary, demonstrating the robustness of the asymptotic normality with respect to different learning rates. The averaged SGDM with large momentum permits a broader selection of learning rates and exhibits reduced sensitivity to their variation.

\subsection{Simulation: logistic regression}

In Example \ref{ex:logistic}, we generate the data $a_i\in\mR^d$ i.i.d from $\mathcal{N}(0_d,I_d)$, and $b_i\in\{0,1\}$ where generated by $b_i=1$ with probability $p_x(a_i)$ and $b_i=0$ otherwise. Here $p_x(a) = 1/(1+\exp(-x^\top a))$, and we use $x=\frac{1}{\sqrt{d}}(1,1,\cdots,1)^\top$. 
We fix the sample size $N=20,000$ and the dimension $d=10$. We set the regularization parameter $\nu$ to zero. Before each simulation, we run the full-batch gradient descent to get the minimizer $x^*$, then compute the Hessian matrix at $x^*$,
\begin{eqnarray*}
    \Sigma = \frac{1}{N}\sum_{i=1}^N p_{x^*}(a_i)(1-p_{x^*}(a_i)) a_i a_i^\top.
\end{eqnarray*}
We specify the batch size $\batch=0.2N$ and the learning rate is $\alpha=0.5$. We consider SGDM with fixed momentum weights $\gamma=0.3,0.5,0.7,0.8,0.9$, as well as the adaptive momentum weight as (\ref{gammaset}). Under the data generation procedure, the average of the adaptive momentum weight is $0.75$.

\begin{figure}[ht]
  \centering
  \subfigure[Small $\gamma$]{
    \includegraphics[width=6.7cm]{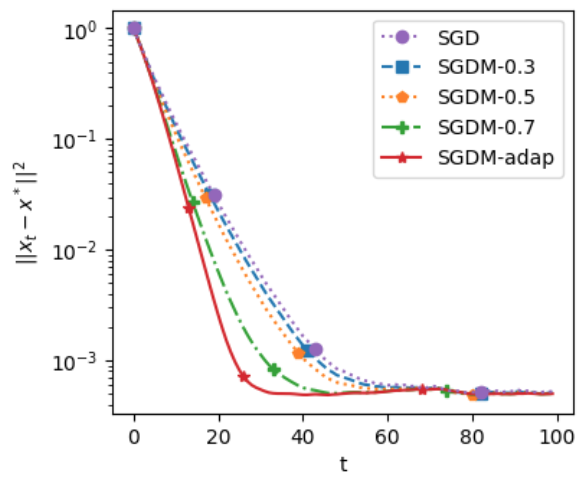}}
  \hspace{0in} 
  \subfigure[Large $\gamma$]{
    \includegraphics[width=6.7cm]{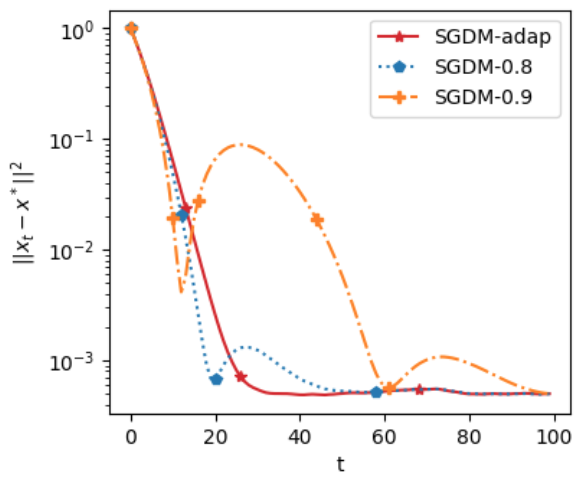}}
  \caption{Performance of SGD and SGDM on the logistic loss, and the average value of the adaptive momentum weight is 0.75.}
    \label{fig:logistic}
\end{figure}

Figure \ref{fig:logistic} illustrates that for SGDM with momentum weights $\gamma$ less than 0.5, the convergence rate of SGDM is slightly faster than SGD, while for $\gamma=0.7$ and adaptive weights $\gamma$ specified as in (\ref{gammaset}), the convergence is much faster than SGD. For $\gamma=0.8$, the convergence is similar to that of the adaptive weight but less stable. For $\gamma=0.9$, the convergence is much slower and more unstable.

 \begin{figure}[ht]
  \centering
  \subfigure[Small $\gamma$]{
    \includegraphics[width=6.7cm]{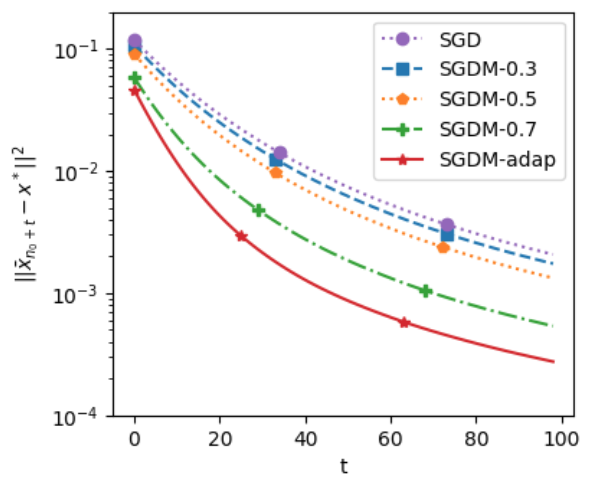}}
  \hspace{0in} 
  \subfigure[Large $\gamma$]{
    \includegraphics[width=6.7cm]{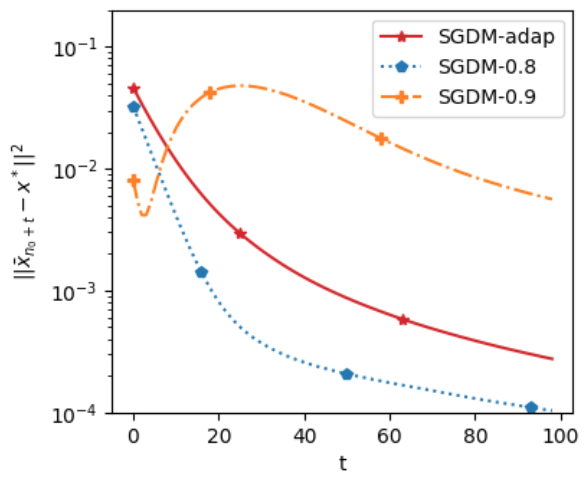}}
  \caption{Performance of Averaged SGD and Averaged SGDM on the logistic loss. The average value of the adaptive momentum weight is 0.75.}
    \label{fig:logistic_ave}
\end{figure}

\begin{figure}[ht]
  \centering
  \subfigure[Small $\gamma$]{
    \includegraphics[width=6.7cm]{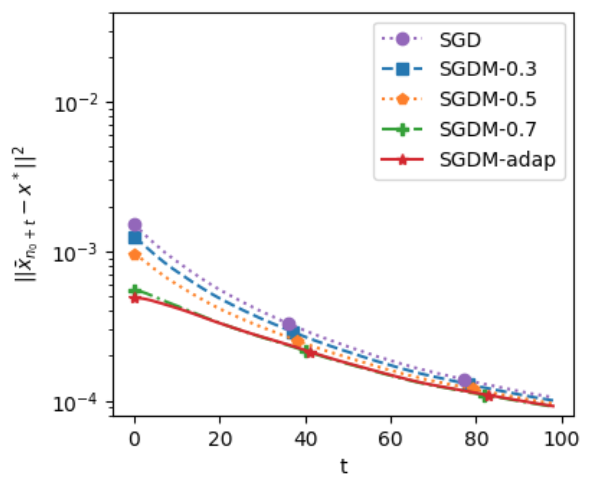}}
  \hspace{0in} 
  \subfigure[Large $\gamma$]{
    \includegraphics[width=6.7cm]{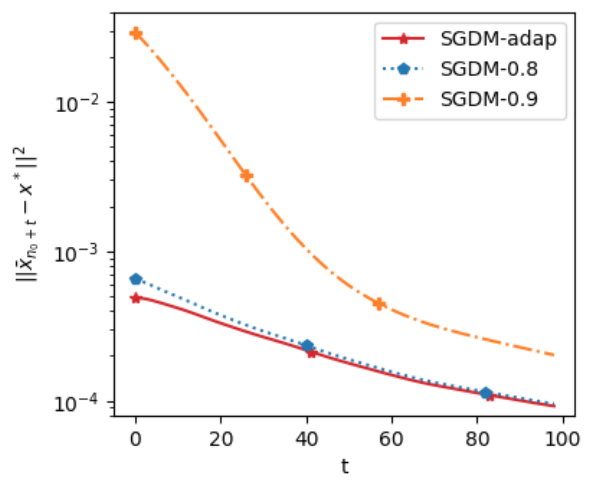}}
  \caption{Performance of Averaged SGD and Averaged SGDM on the logistic loss with $n_0=40$.}
    \label{fig:logistic_ave2}
\end{figure}

To further compare the convergence of averaged SGD and SGDM, we set $n_0$ in Theorem \ref{th2} as $n_0=10$ and $40$. We plot the comparison of the $\ell_2$ error of the averaged SGDM and SGD. 
Figure \ref{fig:logistic_ave}(a) shows that for momentum weights smaller and equal to the adaptive momentum weight in (\ref{gammaset}), the averaged SGDM converges faster than the averaged SGD. In Figure \ref{fig:logistic_ave}(b), we can see the averaged SGDM with $\gamma=0.8$ converges even faster than the averaged SGDM with adaptive $\gamma$, while for $\gamma=0.9$, the averaged SGDM converges faster at the beginning but fluctuates, which cause the increase of $\ell_2$ error for $\gamma=0.9$. 
In comparison with Figure \ref{fig:linear}(b) for the quadratic loss, we can see that although large momentum weight causes fluctuation in both cases, the convergence is more flattened for the the logistic loss in Figure \ref{fig:logistic}(b). This may also explain why the averaged SGDM performs well with $\gamma$ slightly larger than the adaptive momentum weight in Figure \ref{fig:logistic_ave}(b). Figure \ref{fig:logistic_ave2} demonstrates that with $n_0=40$, there is an improvement in the bias order from $10^{-3}$, as seen in Figure \ref{fig:logistic}, to $10^{-4}$. This improvement reflects a sublinear convergence that is in agreement with the behavior observed in Figure \ref{fig:average2} for quadratic loss. Consequently, for averaged SGDM, this suggests reduced sensitivity to the choice of a smaller $n_0$. 




\subsection{Real data: MNIST classification}
In this section, we consider the multinomial logistic regression on the MNIST dataset, which consists of $N=60,000$ images of handwritten digits with size $28\times 28$.  
We reshape the images to vectors of size $784\times 1$. For the samples $(a_i, b_i)$, where $a_i\in\mR^{784}$ are the vectorized images and $b_i\in\mR^{10}$ are one-hot indicators corresponding the digits $0,1,\cdots,9$, the loss function is  $
    f_i(X) = H(X^\top a_i,b_i),$
where $X\in\mR^{784\times 10}$ and $H$ is the multi-class cross entropy.
We specify the batch size $\batch=256$ and the learning rate $\alpha=1.0$. Due to the difficulty of computing the condition number $L/\mu$, we fix the momentum weight $\gamma=0.1,0.3,0.5,0.7,0.9,0.99$ in training. 

 \begin{figure}[ht]
  \centering
  \subfigure[Small $\gamma$]{
    \includegraphics[width=6.7cm]{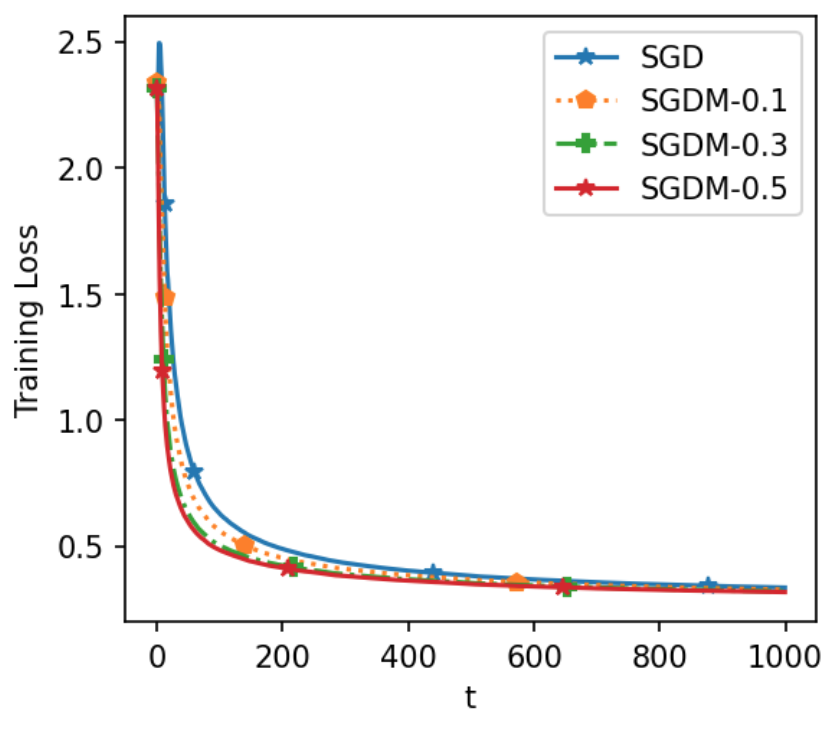}}
  \hspace{0in} 
  \subfigure[Large $\gamma$]{
    \includegraphics[width=6.7cm]{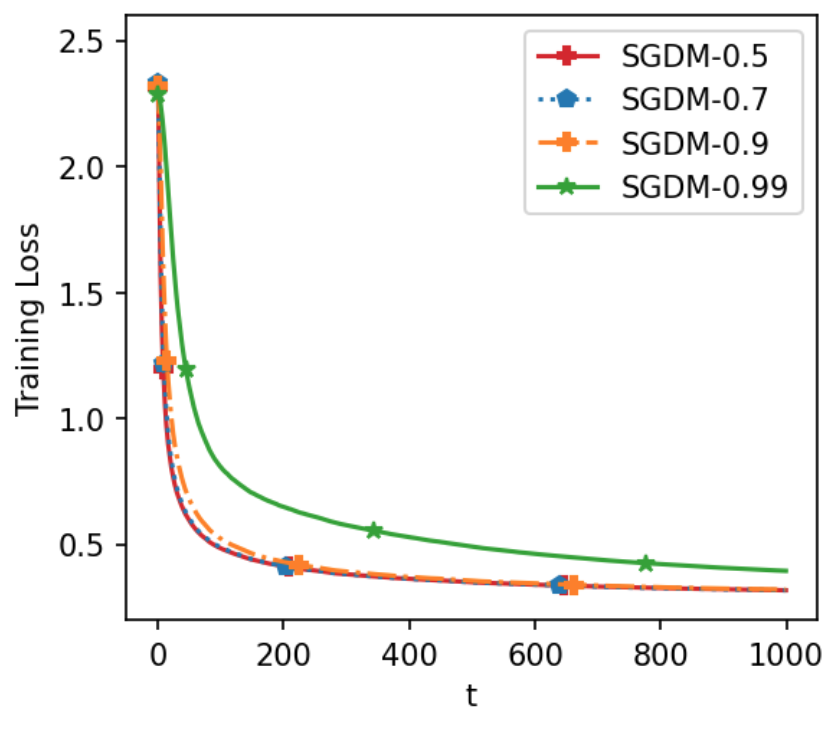}}
  \caption{Performance of SGD and SGDM on MNIST.}
    \label{fig:mnist}
\end{figure}


Figure \ref{fig:mnist} shows the convergence of the training loss over iterations $t$. For the purpose of clear representation, the reported training loss for each iteration $t$ is based on an average of mini-batch stochastic losses $g_{\eta_t}$ of the past $\lfloor N/\batch\rfloor$ batches, with the batch size $\batch$. 

From Figure \ref{fig:mnist}, we see that 
SGDM with momentum weight $\gamma$ ranging from $0.1$ to $0.9$ greatly outperforms SGD. SGDM with $\gamma=0.5$ converges fastest in the earlier iterations, while the training losses are almost identical for different momentum weights in the later iterations, except for $\gamma=0.99$. The convergence rate is not very sensitive to the momentum weights and a wide range of $\gamma\in[0.3,0.9]$ leads to similar performance.

\section{Conclusion}\label{sec:conclution}

Our study sheds light on the performance of mini-batch SGDM in solving optimization problems under strongly convex loss functions. The convergence rate of mini-batch SGDM is influenced by several factors, including the batch size, the momentum weight, and the learning rate. Our analysis rigorously shows that certain choices of momentum weight with a reasonably large batch size can lead to faster convergence compared to SGD. This finding is consistent with previous numerical studies on SGDM.  Additionally, our findings, supported by theoretical analysis and numerical experiments, indicate that SGDM permits a broader selection of learning rates.

We further establish the asymptotic normality of averaged SGDM in the quadratic settings and reveal the non-vanishing bias of that in the general settings. Our investigation reveals that averaged SGDM is asymptotically equivalent to averaged SGD. By minimizing the remainder term, we give the optimal learning rate and the corresponding rate of convergence to asymptotic normality. In addition, we present the asymptotic covariance matrix for the averaged SGDM, enabling the uncertainty quantification of the algorithm outputs and statistical inference of the true model parameters based on SGDM, as opposed to SGD. 

In summary, our study contributes to the theoretical understanding of mini-batch SGDM and has practical implications for developing efficient optimization algorithms in machine learning. It is noteworthy to mention that our study assumed that the loss function was smooth and strongly convex. In future work, our work can be extended in several ways. It would be interesting to investigate the performance of mini-batch SGDM under non-convex loss functions, since this would have important implications for the application of SGDM in deep learning. It is also of interest to extend the analysis to the case when the learning rate decays over time.

\newpage




\vskip 0.2in

\bibliographystyle{chicago}
\bibliography{references}

\begin{thebibliography}{}

\bibitem[\protect\citeauthoryear{Bollapragada, Chen, and Ward}{Bollapragada et~al.}{2022}]{bollapragada2022fast}
Bollapragada, R., T.~Chen, and R.~Ward (2022).
\newblock On the fast convergence of minibatch heavy ball momentum.
\newblock {\em arXiv preprint arXiv:2206.07553\/}.

\bibitem[\protect\citeauthoryear{Bottou, Curtis, and Nocedal}{Bottou et~al.}{2018}]{bottou2018optimization}
Bottou, L., F.~E. Curtis, and J.~Nocedal (2018).
\newblock Optimization methods for large-scale machine learning.
\newblock {\em SIAM Review\/}~{\em 60\/}(2), 223--311.

\bibitem[\protect\citeauthoryear{Chen, Lai, Li, and Zhang}{Chen et~al.}{2023}]{chen2023online}
Chen, X., Z.~Lai, H.~Li, and Y.~Zhang (2023).
\newblock Online statistical inference for stochastic optimization via {K}iefer-{W}olfowitz methods.
\newblock {\em Journal of the American Statistical Association\/}~(To appear).

\bibitem[\protect\citeauthoryear{Chen, Lee, Tong, and Zhang}{Chen et~al.}{2020}]{chen2020statistical}
Chen, X., J.~D. Lee, X.~T. Tong, and Y.~Zhang (2020).
\newblock {Statistical inference for model parameters in stochastic gradient descent}.
\newblock {\em The Annals of Statistics\/}~{\em 48\/}(1), 251 -- 273.

\bibitem[\protect\citeauthoryear{Defazio}{Defazio}{2020}]{defazio2020understanding}
Defazio, A. (2020).
\newblock Understanding the role of momentum in non-convex optimization: Practical insights from a lyapunov analysis.
\newblock {\em arXiv preprint arXiv:2010.00406\/}.

\bibitem[\protect\citeauthoryear{Gitman, Lang, Zhang, and Xiao}{Gitman et~al.}{2019}]{gitman2019understanding}
Gitman, I., H.~Lang, P.~Zhang, and L.~Xiao (2019).
\newblock Understanding the role of momentum in stochastic gradient methods.
\newblock {\em Advances in Neural Information Processing Systems\/}~{\em 32}.

\bibitem[\protect\citeauthoryear{Jin, Xing, and He}{Jin et~al.}{2022}]{jin2022convergence}
Jin, R., Y.~Xing, and X.~He (2022).
\newblock On the convergence of {mSGD} and {AdaGrad} for stochastic optimization.
\newblock {\em arXiv preprint arXiv:2201.11204\/}.

\bibitem[\protect\citeauthoryear{Kidambi, Netrapalli, Jain, and Kakade}{Kidambi et~al.}{2018}]{kidambi2018insufficiency}
Kidambi, R., P.~Netrapalli, P.~Jain, and S.~Kakade (2018).
\newblock On the insufficiency of existing momentum schemes for stochastic optimization.
\newblock In {\em Information Theory and Applications Workshop}, pp.\  1--9.

\bibitem[\protect\citeauthoryear{Kingma and Ba}{Kingma and Ba}{2014}]{kingma2014adam}
Kingma, D.~P. and J.~Ba (2014).
\newblock Adam: A method for stochastic optimization.
\newblock {\em arXiv preprint arXiv:1412.6980\/}.

\bibitem[\protect\citeauthoryear{Lee, Cheng, Paquette, and Paquette}{Lee et~al.}{2022}]{lee2022trajectory}
Lee, K., A.~Cheng, E.~Paquette, and C.~Paquette (2022).
\newblock Trajectory of mini-batch momentum: Batch size saturation and convergence in high dimensions.
\newblock {\em Advances in Neural Information Processing Systems\/}~{\em 35}, 36944--36957.

\bibitem[\protect\citeauthoryear{Lee, Liao, Seo, and Shin}{Lee et~al.}{2022}]{lee2022fast}
Lee, S., Y.~Liao, M.~H. Seo, and Y.~Shin (2022).
\newblock Fast and robust online inference with stochastic gradient descent via random scaling.
\newblock In {\em Proceedings of the AAAI Conference on Artificial Intelligence}, Volume~36, pp.\  7381--7389.

\bibitem[\protect\citeauthoryear{Li, Liu, and Orabona}{Li et~al.}{2022}]{li2022last}
Li, X., M.~Liu, and F.~Orabona (2022).
\newblock On the last iterate convergence of momentum methods.
\newblock In {\em International Conference on Algorithmic Learning Theory}, pp.\  699--717.

\bibitem[\protect\citeauthoryear{Liu, Gao, and Yin}{Liu et~al.}{2020}]{liu2020improved}
Liu, Y., Y.~Gao, and W.~Yin (2020).
\newblock An improved analysis of stochastic gradient descent with momentum.
\newblock {\em Advances in Neural Information Processing Systems\/}~{\em 33}, 18261--18271.

\bibitem[\protect\citeauthoryear{Loizou and Richt{\'a}rik}{Loizou and Richt{\'a}rik}{2017}]{loizou2017linearly}
Loizou, N. and P.~Richt{\'a}rik (2017).
\newblock Linearly convergent stochastic heavy ball method for minimizing generalization error.
\newblock {\em arXiv preprint arXiv:1710.10737\/}.

\bibitem[\protect\citeauthoryear{Loizou and Richt{\'a}rik}{Loizou and Richt{\'a}rik}{2020}]{loizou2020momentum}
Loizou, N. and P.~Richt{\'a}rik (2020).
\newblock Momentum and stochastic momentum for stochastic gradient, newton, proximal point and subspace descent methods.
\newblock {\em Computational Optimization and Applications\/}~{\em 77\/}(3), 653--710.

\bibitem[\protect\citeauthoryear{Mai and Johansson}{Mai and Johansson}{2020}]{mai2020convergence}
Mai, V. and M.~Johansson (2020).
\newblock Convergence of a stochastic gradient method with momentum for non-smooth non-convex optimization.
\newblock In {\em International Conference on Machine Learning}, pp.\  6630--6639.

\bibitem[\protect\citeauthoryear{Moulines and Bach}{Moulines and Bach}{2011}]{moulines2011non}
Moulines, E. and F.~Bach (2011).
\newblock Non-asymptotic analysis of stochastic approximation algorithms for machine learning.
\newblock {\em Advances in Neural Information Processing Systems\/}~{\em 24}.

\bibitem[\protect\citeauthoryear{Nesterov}{Nesterov}{1983}]{nesterov1983method}
Nesterov, Y. (1983).
\newblock A method for unconstrained convex minimization problem with the rate of convergence $o (1/k^{2})$.
\newblock In {\em Doklady AN USSR}, Volume 269, pp.\  543--547.

\bibitem[\protect\citeauthoryear{Nesterov}{Nesterov}{2003}]{nesterov2003introductory}
Nesterov, Y. (2003).
\newblock {\em Introductory lectures on convex optimization: A basic course}, Volume~87.
\newblock Springer.

\bibitem[\protect\citeauthoryear{Nguyen, Nguyen, Dijk, Richt{\'a}rik, Scheinberg, and Tak{\'a}c}{Nguyen et~al.}{2018}]{nguyen2018sgd}
Nguyen, L., P.~H. Nguyen, M.~Dijk, P.~Richt{\'a}rik, K.~Scheinberg, and M.~Tak{\'a}c (2018).
\newblock {SGD} and {H}ogwild! convergence without the bounded gradients assumption.
\newblock In {\em International Conference on Machine Learning}, pp.\  3750--3758.

\bibitem[\protect\citeauthoryear{Nocedal and Wright}{Nocedal and Wright}{2006}]{nocedal2006numerical}
Nocedal, J. and S.~J. Wright (2006).
\newblock Numerical optimization. springer series in operations research.
\newblock {\em {SIAM} J Optimization\/}.

\bibitem[\protect\citeauthoryear{Paquette and Paquette}{Paquette and Paquette}{2021}]{paquette2021dynamics}
Paquette, C. and E.~Paquette (2021).
\newblock Dynamics of stochastic momentum methods on large-scale, quadratic models.
\newblock {\em Advances in Neural Information Processing Systems\/}~{\em 34}, 9229--9240.

\bibitem[\protect\citeauthoryear{Polyak}{Polyak}{1964}]{polyak1964some}
Polyak, B.~T. (1964).
\newblock Some methods of speeding up the convergence of iteration methods.
\newblock {\em USSR Computational Mathematics and Mathematical Physics\/}~{\em 4\/}(5), 1--17.

\bibitem[\protect\citeauthoryear{Polyak and Juditsky}{Polyak and Juditsky}{1992}]{polyak1992acceleration}
Polyak, B.~T. and A.~B. Juditsky (1992).
\newblock Acceleration of stochastic approximation by averaging.
\newblock {\em SIAM Journal on Control and Optimization\/}~{\em 30\/}(4), 838--855.

\bibitem[\protect\citeauthoryear{Richt{\'a}rik and Tak{\'a}c}{Richt{\'a}rik and Tak{\'a}c}{2020}]{richtarik2020stochastic}
Richt{\'a}rik, P. and M.~Tak{\'a}c (2020).
\newblock Stochastic reformulations of linear systems: algorithms and convergence theory.
\newblock {\em SIAM Journal on Matrix Analysis and Applications\/}~{\em 41\/}(2), 487--524.

\bibitem[\protect\citeauthoryear{Robbins and Monro}{Robbins and Monro}{1951}]{robbins1951stochastic}
Robbins, H. and S.~Monro (1951).
\newblock A stochastic approximation method.
\newblock {\em The Annals of Mathematical Statistics\/}, 400--407.

\bibitem[\protect\citeauthoryear{Ruppert}{Ruppert}{1988}]{ruppert1988efficient}
Ruppert, D. (1988).
\newblock Efficient estimations from a slowly convergent {R}obbins-{M}onro process.
\newblock Technical report, Cornell University Operations Research and Industrial Engineering.

\bibitem[\protect\citeauthoryear{Sebbouh, Gower, and Defazio}{Sebbouh et~al.}{2021}]{sebbouh2021almost}
Sebbouh, O., R.~M. Gower, and A.~Defazio (2021).
\newblock Almost sure convergence rates for stochastic gradient descent and stochastic heavy ball.
\newblock In {\em Conference on Learning Theory}, pp.\  3935--3971. PMLR.

\bibitem[\protect\citeauthoryear{Su and Zhu}{Su and Zhu}{2023}]{su2023higrad}
Su, W.~J. and Y.~Zhu (2023).
\newblock Higrad: Uncertainty quantification for online learning and stochastic approximation.
\newblock {\em Journal of Machine Learning Research\/}~{\em 24\/}(124), 1--53.

\bibitem[\protect\citeauthoryear{Toulis, Horel, and Airoldi}{Toulis et~al.}{2021}]{toulis2021proximal}
Toulis, P., T.~Horel, and E.~M. Airoldi (2021).
\newblock The proximal {R}obbins-{M}onro method.
\newblock {\em Journal of the Royal Statistical Society Series B: Statistical Methodology\/}~{\em 83\/}(1), 188--212.

\bibitem[\protect\citeauthoryear{Wang and Johansson}{Wang and Johansson}{2022}]{wang2022uniform}
Wang, X. and M.~Johansson (2022).
\newblock On uniform boundedness properties of sgd and its momentum variants.
\newblock {\em arXiv preprint arXiv:2201.10245\/}.

\bibitem[\protect\citeauthoryear{Yan, Yang, Li, Lin, and Yang}{Yan et~al.}{2018}]{yan2018unified}
Yan, Y., T.~Yang, Z.~Li, Q.~Lin, and Y.~Yang (2018).
\newblock A unified analysis of stochastic momentum methods for deep learning.
\newblock {\em arXiv preprint arXiv:1808.10396\/}.

\bibitem[\protect\citeauthoryear{Yang, Lin, and Li}{Yang et~al.}{2016}]{yang2016unified}
Yang, T., Q.~Lin, and Z.~Li (2016).
\newblock Unified convergence analysis of stochastic momentum methods for convex and non-convex optimization.
\newblock {\em arXiv preprint arXiv:1604.03257\/}.

\bibitem[\protect\citeauthoryear{Zeiler}{Zeiler}{2012}]{zeiler2012adadelta}
Zeiler, M.~D. (2012).
\newblock Adadelta: an adaptive learning rate method.
\newblock {\em arXiv preprint arXiv:1212.5701\/}.

\bibitem[\protect\citeauthoryear{Zhu, Chen, and Wu}{Zhu et~al.}{2023}]{zhu2023online}
Zhu, W., X.~Chen, and W.~B. Wu (2023).
\newblock Online covariance matrix estimation in stochastic gradient descent.
\newblock {\em Journal of the American Statistical Association\/}~{\em 118\/}(541), 393--404.

\bibitem[\protect\citeauthoryear{Zhu and Dong}{Zhu and Dong}{2021}]{zhu2021constructing}
Zhu, Y. and J.~Dong (2021).
\newblock On constructing confidence region for model parameters in stochastic gradient descent via batch means.
\newblock In {\em 2021 Winter Simulation Conference (WSC)}, pp.\  1--12.

\end{thebibliography}

\newpage
\begin{appendices}
\section{Proofs of finite-time convergence rates of SGDM}
\subsection{Proof of Theorem \ref{thm:g_gamma}}

Recall that
\begin{eqnarray*}
m_{t+1}=\gamma m_{t}+(1-\gamma)\triangledown g_{\eta_{t}}(x_{t}),\quad
x_{t+1}=x_{t}-\alpha m_{t+1}.
\end{eqnarray*} 
Put
\begin{eqnarray*}
\Gamma=\left(
    \begin{array}{ccc}
       \gamma I, & (1-\gamma)\Sigma\\
      -\alpha \gamma I,  & I-\alpha(1-\gamma)\Sigma \\
    \end{array}
\right).
\end{eqnarray*}
Let $\widetilde{x}_{t}=x_{t}-x^{*}$ and $\widetilde{m}_{t+1} = (1-\gamma) \sum_{j=1}^{t}\gamma^{t-j}\Sigma \widetilde{x}_j$, we can write 
\begin{eqnarray}\label{iter}
    \notag\left(
    \begin{array}{ccc}
      \widetilde{m}_{t+1} \\
      \widetilde{x}_{t+1} \\
    \end{array}
\right)&& =\Gamma\left(
    \begin{array}{ccc}
      \widetilde{m}_{t} \\
      \widetilde{x}_{t} \\
    \end{array}
\right)-\alpha (1-\gamma) \left(
    \begin{array}{ccc}
    0 \\
      \sum_{j=1}^{t}\gamma^{t-j}(\triangledown g_{\eta_{j}}(x_j) -\Sigma\widetilde{x}_j)\\
    \end{array}
\right)\\
&&=\Gamma^{t}\left(
    \begin{array}{ccc}
      \widetilde{m}_{1} \\
      \widetilde{x}_{1} \\
    \end{array}
\right)-\alpha (1-\gamma) \sum_{j=1}^t \Gamma^{t-j}\left(
    \begin{array}{ccc}
    0 \\
      \sum_{k=1}^{j}\gamma^{j-k}(\triangledown g_{\eta_{k}}(x_k)-\Sigma\widetilde{x}_k)\\
    \end{array}
\right).
\end{eqnarray}
Define 
\begin{eqnarray*}
    q_t:= \mE[\norm{\widetilde{m}_t}^2+\norm{\widetilde{x}_t}^2].
\end{eqnarray*}
Recall for $\overline{L}=0$, there hold that $\triangledown g(x) = \mE[\triangledown g_\eta (x)] = \Sigma (x-x^*)$ and
\begin{eqnarray*}
    \triangledown g_{\eta_{k}}(x_k)-\Sigma\widetilde{x}_k = \bracket{\triangledown g_{\eta_{k}}(x_k) - \triangledown g_{\eta_{k}}(x^*) + \triangledown g(x^*) - \triangledown g(x_k)} + \triangledown g_{\eta_{k}}(x^*).
\end{eqnarray*}
By the $L_f$-smooth of the individual function $f_\xi(x)$ in (A3) and the independence of $L_\xi$ and $\widetilde{x}_k$, we have
\begin{eqnarray*}
    \mE\norm{\triangledown f_{\xi}(x_k) - \triangledown f_{\xi}(x^*) + \triangledown f(x^*) - \triangledown f(x_k)}^2\leq\mE\norm{\triangledown f_{\xi}(x_k) - \triangledown f_{\xi}(x^*)}^2\leq  L_f^2 \mE\norm{\widetilde{x}_k}^2.
\end{eqnarray*}
Then from the iteration (\ref{iter}), we have
\begin{eqnarray}
    \notag q_{t+1}&&\leq  \norm{\Gamma^t}^2 q_1 + 2\alpha^2 (1-\gamma)^2 \frac{\sigma^2}{\batch}\sum_{k=1}^t \bracket{\sum_{j=k}^t \norm{\Gamma^{t-j}}\gamma^{j-k}}^2  \\
    \label{l0qt}&&+ 2\alpha^2 (1-\gamma)^2\frac{L_f^2}{\batch} \sum_{k=1}^t\bracket{\sum_{j=k}^t \norm{\Gamma^{t-j}}\gamma^{j-k}}^2 q_k.
\end{eqnarray}

In the remaining part of proof, we use the inductive method to derive the bound of $q_t$. Let $C_1\geq 1$ and $C_2\geq 1$ be some constants which will be specified later. We will prove that, for any $t\geq 1$, the inequalities
\begin{eqnarray}\label{induction}
    q_{j}\leq \frac{C_1}{\batch(1-\lambda)} \alpha^2\sigma^2+C_2q_1\lambda^{2(1-\delta)(j-1)},\quad 1\leq j\leq t
\end{eqnarray}
with fixed $\delta\in(0,1]$, imply that 
\begin{eqnarray}\label{induction2}
    q_{t+1}\leq \frac{C_1}{\batch(1-\lambda)} \alpha^2\sigma^2+C_2q_1\lambda^{2(1-\delta)t}.
\end{eqnarray}

By Lemma \ref{lemma:g_gamma} below, we have $\norm{\Gamma^j}\leq M \lambda^j$ for all $j\geq 0$.
Theorem \ref{thm:lambda} provides that $\lambda\geq \sqrt{\gamma}$. By applying Lemma \ref{le:sum},  we are equipped to manage the summation presented in (\ref{l0qt}).
Consequently, from (\ref{l0qt}) and (\ref{induction}), we deduce that
\begin{eqnarray*}
    q_{t+1}&&\leq M^2\lambda^{2t} q_1 +2\alpha^2  \frac{\sigma^2}{\batch} M^2 \frac{2}{1-\lambda} \\
    &&+ 2\alpha^2 (1-\gamma)^2 \frac{L_f^2}{\batch} M^2\bracket{ \frac{2C_1\alpha^2\sigma^2}{\batch(1-\gamma)^2(1-\lambda)^2} + C_2 q_1 \lambda^{2(1-\delta)t}\frac{1}{\lambda^{2(1-\delta)}}\frac{4}{\delta(1-\gamma)^2(1-\lambda)}}\\
    &&\leq \frac{C_1}{\batch(1-\lambda)} \alpha^2\sigma^2M^2 \bracket{\frac{4}{C_1} + \frac{4 \alpha^2L_f^2}{\batch(1-\lambda)}}+ C_2 q_1 \lambda^{2(1-\delta)t}M^2 \bracket{\frac{\lambda^{2\delta t}}{C_2} + \frac{8\alpha^2  L_f^2}{\batch\lambda^{2(1-\delta)}\delta(1-\lambda)}}.
\end{eqnarray*}
Let the learning rate $\alpha$ and the batch size $\batch$ satisfy
\begin{eqnarray*}
    8 M^2  \frac{1}{\lambda^{2(1-\delta)}} \frac{1}{\delta(1-\lambda)} \alpha^2\frac{L_f^2}{\batch}\leq \frac{1}{2}, 
\end{eqnarray*}
we can take
\begin{eqnarray*}
    C_1  = 8M^2,\quad C_2  = 2M^2.
\end{eqnarray*}
It is straightforward to verify that (\ref{induction2}) is satisfied, and the induction is completed.

\subsection{Proof of Theorem \ref{thm:lambda}}
Put
\begin{eqnarray*}
\Gamma=\left(
    \begin{array}{ccc}
       U, & 0_d\\
      0_d,  & U \\
    \end{array}
\right)\left(
    \begin{array}{ccc}
       \gamma I, & (1-\gamma)A \\
      -\alpha \gamma I,  & I-\alpha(1-\gamma)A \\
    \end{array}
\right)\left(
    \begin{array}{ccc}
       U^{\top}, & 0_d\\
      0_d,  & U^{\top} \\
    \end{array}
\right)
\end{eqnarray*}
where $\Sigma=U diag(\kappa_{1},...,\kappa_{d})U^{T}=:UAU^{T}$, $\mu=\kappa_{1}\leq \cdots\leq \kappa_{d}=L$, $U$ is an orthogonal matrix. Now we can define
\begin{eqnarray*}
\lambda^{\pm}_{k}=\frac{\gamma+1-\alpha(1-\gamma)\kappa_{k}\pm\sqrt{(\alpha(1-\gamma)\kappa_{k}-\gamma-1)^{2}-4\gamma}}{2},
\end{eqnarray*}
for $1\leq k\leq d$, and the diagonal matrix
\begin{eqnarray}\label{def:Lambda}
\Lambda=\left(
    \begin{array}{ccc}
       diag(\lambda^{+}_{1},...,\lambda^{+}_{d}), & 0_d \\
      0_d,  & diag(\lambda^{-}_{1},...,\lambda^{-}_{d})\\
    \end{array}
\right).
\end{eqnarray}
Notice that if $(\alpha(1-\gamma)\kappa_{k}-\gamma-1)^{2}<4\gamma$, $\lambda^{\pm}_{k}$ is complex and the definition is 
\begin{eqnarray*}
    \lambda^{\pm}_{k}=\frac{\gamma+1-\alpha(1-\gamma)\kappa_{k}\pm \sqrt{-1}\sqrt{4\gamma-(\alpha(1-\gamma)\kappa_{k}-\gamma-1)^{2}}}{2}.
\end{eqnarray*}

Prior to establishing Theorem \ref{thm:lambda}, we first demonstrate the diagonalization of $\Gamma$, with $\Lambda$ representing the resultant diagonal matrix. Subsequent to this, we proceed to delineate the bounds of the spectral radius of $\Lambda$.

\begin{lemma}\label{lemma:g_gamma} For $\alpha, \gamma,L$ satisfying that $\alpha L<2(1+\gamma)/(1-\gamma)$ and $\abs{\gamma+1-\alpha(1-\gamma)\kappa_{k}}\neq 2\sqrt{\gamma}$ for all $1\leq k\leq d$, we have
\begin{eqnarray*}
\left(
    \begin{array}{ccc}
       \gamma I, & (1-\gamma)A \\
      -\alpha \gamma I,  & I-\alpha(1-\gamma)A \\
    \end{array}
\right)=P\Lambda P^{-1},
\end{eqnarray*}
for some invertible matrix $P$ that satisfies
\begin{eqnarray*}
\norm{P}\leq 2, \quad \norm{P^{-1}}\leq \frac{2}{\sqrt{\Delta}} \bracket{2(1-\gamma)(1+\alpha L+ L)+3\alpha\gamma},
\end{eqnarray*}
where $\Delta = \min_k\bbracket{\abs{\bracket{\gamma+1-\alpha(1-\gamma)\kappa_{k}}^2-4\gamma}}$. Therefore, $\norm{\Gamma^{j}}\leq M \lambda^j$ for any $j\geq 1$, where
\begin{eqnarray*}
    M = \frac{4}{\sqrt{\Delta}} \bracket{2(1-\gamma)(1+\alpha L+ L)+3\alpha\gamma},
\end{eqnarray*}
and $\lambda$ is the spectral radius of $\Lambda$.
\end{lemma}

\begin{proof}
The eigenvalues satisfy
\begin{eqnarray*}
&&(\lambda^{\pm}_{k})^{2}+(-\gamma-1+\alpha(1-\gamma)\kappa_{k})\lambda^{\pm}_{k}+\gamma=0,\\
&&\lambda^{+}_{k} + \lambda^{-}_{k} = \gamma+1-\alpha(1-\gamma)\kappa_{k},\quad \lambda^{+}_{k} \lambda^{-}_{k} = \gamma.
\end{eqnarray*}
Given that $\alpha L < 2(1+\gamma)/(1-\gamma)$ and considering $\mu \leq \kappa_k \leq L$, it follows that $|\gamma + 1 - \alpha(1 - \gamma)\kappa_{k}| < 1 + \gamma$ for all $k$. 
Considering the product of the eigenvalues $\lambda^{+}_{k} \lambda^{-}_{k} = \gamma < 1$, subsequent calculations reveal that $\max_k\{|\lambda^{\pm}_{k}|\} < 1$, which holds even in the case of complex eigenvalues. Moreover, since $|\gamma + 1 - \alpha(1 - \gamma)\kappa_{k}| \neq 2\sqrt{\gamma}$ for any $k$, it follows that $\lambda^{+}_{k} \neq \lambda^{-}_{k}$, confirming that the matrix is indeed diagonalizable.




Let $e_{k}$ be the unit vector which has 1 in the $k$-th coordinate and others zero. By the definition of $A$, we have
\begin{eqnarray*}
(\alpha\gamma(1-\gamma)+\alpha(\lambda^{\pm}_{k}-\gamma)(1-\gamma))Ae_{k}=(\lambda^{\pm}_{k}-\gamma)(1-\lambda^{\pm}_{k})e_{k},
\end{eqnarray*}
which is equivalent to
\begin{eqnarray}\label{eq6}
(1-\gamma)Ae_{k}=\frac{\lambda^{\pm}_{k}-\gamma}{\alpha\gamma}\Big{\{}(1-\lambda^{\pm}_{k})e_{k}-\alpha(1-\gamma)Ae_{k}\Big{\}}.
\end{eqnarray}
Now define $\z^{\pm}_{k}$ by
\begin{eqnarray*}
(\lambda^{\pm}_{k}-\gamma)\z^{\pm}_{k}=(1-\gamma)Ae_{k},
\end{eqnarray*}
which can be written as $\gamma\z_{k}^\pm+(1-\gamma)Ae_{k}=\lambda^{\pm}_{k}\z_{k}^\pm$.
Combining this equation with (\ref{eq6}), we have
\begin{eqnarray*}
(\lambda^{\pm}_{k}-\gamma)\z^{\pm}_{k}=\frac{\lambda^{\pm}_{k}-\gamma}{\alpha\gamma}\Big{\{}(1-\lambda^{\pm}_{k})e_{k}-\alpha(1-\gamma)Ae_{k}\Big{\}},
\end{eqnarray*}
which yields that
\begin{eqnarray*}
-\alpha\gamma\z^{\pm}_{k}+e_{k}-\alpha(1-\gamma)Ae_{k}=\lambda^{\pm}_{k}e_{k}.
\end{eqnarray*}
Thus
\begin{eqnarray}\label{eq7}
\left(
    \begin{array}{ccc}
       \gamma I, & (1-\gamma)A \\
      -\alpha \gamma I,  & I-\alpha(1-\gamma)A \\
    \end{array}
\right)\left(
    \begin{array}{c}
       \z^{\pm}_{k}\\
      e_{k}\\
    \end{array}
\right)=\lambda^{\pm}_{k}\left(
    \begin{array}{c}
       \z^{\pm}_{k}\\
      e_{k}\\
    \end{array}
\right)=\left(
    \begin{array}{c}
       \z^{\pm}_{k}\\
      e_{k}\\
    \end{array}
\right)\lambda^{\pm}_{k}.
\end{eqnarray}

Define
\begin{eqnarray*}
\z^{+}=(\z^{+}_{1},...,\z^{+}_{d}),\quad\z^{-}=(\z^{-}_{1},...,\z^{-}_{d}),
\end{eqnarray*}
and
\begin{eqnarray*}
    P &&= \left(
    \begin{array}{cc}
        \z^+\bracket{\abs{\z^+}^2+I}^{-\frac{1}{2}}, & \z^-\bracket{\abs{\z^-}^2+I}^{-\frac{1}{2}} \\
        I\bracket{\abs{\z^+}^2+I}^{-\frac{1}{2}},& I\bracket{\abs{\z^-}^2+I}^{-\frac{1}{2}}
    \end{array}\right), \\
    P^{-1} &&= \left(
    \begin{array}{cc}
       \bracket{\abs{\z^+}^2+I}^{\frac{1}{2}}(\z^+-\z^-)^{-1}, & -\bracket{\abs{\z^+}^2+I}^{\frac{1}{2}} \z^{-}(\z^+-\z^-)^{-1} \\
       - \bracket{\abs{\z^-}^2+I}^{\frac{1}{2}}(\z^+-\z^-)^{-1}, & \bracket{\abs{\z^-}^2+I}^{\frac{1}{2}} \z^{+}(\z^+-\z^-)^{-1}
    \end{array}\right).
\end{eqnarray*}
Note that $\z^{+}$ and $\z^{-}$ are diagonal matrices. Then by (\ref{eq7}),
\begin{eqnarray*}
\left(
    \begin{array}{ccc}
       \gamma I, & (1-\gamma)A \\
      -\alpha \gamma I,  & I-\alpha(1-\gamma)A \\
    \end{array}
\right)P=P\Lambda,
\end{eqnarray*}
where $\Lambda$ is defined by (\ref{def:Lambda}).
By the definition of $\z^\pm$, we have
\begin{eqnarray*}
    \z^{\pm} &&= diag\bracket{\frac{\lambda_k^\mp-\gamma}{\alpha\gamma}},\\
    \z^{+}\z^{-} &&= diag\bracket{\frac{(1-\gamma)\kappa_k}{\alpha\gamma}},\\
    (\z^{+}-\z^{-})^{-1} &&= diag\bracket{\frac{-\alpha\gamma}{\sqrt{(\gamma+1-\alpha(1-\gamma)\kappa_k)^2-4\gamma}}}.
\end{eqnarray*}
For $\abs{\z_k^+}+\abs{\z_k^-}$, we have
\begin{eqnarray*}
    \abs{\z_k^+}+\abs{\z_k^-}\leq 2\sqrt{\abs{\z_k^+}^2+\abs{\z_k^-}^2} = \left\{\begin{array}{cc}
         \frac{2}{\alpha \gamma}\sqrt{2(\lambda_k^+-\gamma)(\lambda_k^--\gamma)},& (\gamma+1-\alpha(1-\gamma)\kappa_k)^2<4\gamma,  \\
         \frac{2}{\alpha \gamma}\sqrt{(\lambda_k^+-\gamma)^2+(\lambda_k^--\gamma)^2},& (\gamma+1-\alpha(1-\gamma)\kappa_k)^2>4\gamma,
    \end{array}\right.
\end{eqnarray*}
and 
\begin{eqnarray*}
    \frac{2}{\alpha \gamma}\sqrt{2(\lambda_k^+-\gamma)(\lambda_k^--\gamma)} &&= \frac{2}{\alpha\gamma} \sqrt{2\alpha \gamma(1-\gamma)\kappa_k},\\
     \frac{2}{\alpha \gamma}\sqrt{(\lambda_k^+-\gamma)^2+(\lambda_k^--\gamma)^2}&&=\frac{2}{\alpha\gamma} \sqrt{(\lambda_k^++\lambda_k^--\gamma)^2-2\lambda_k^+\lambda_k^-+\gamma^2}\\
    &&=\frac{2}{\alpha\gamma} \sqrt{(1-\gamma)^2-2\alpha(1-\gamma)\kappa_k+\alpha^2(1-\gamma)^2 \kappa_k^2}\leq \frac{2}{\alpha\gamma} (1-\gamma)\bracket{1+\alpha \kappa_k}.
\end{eqnarray*}
So it holds that
\begin{eqnarray*}
    \abs{\z_k^+}+\abs{\z_k^-}&&\leq  \frac{2}{\alpha\gamma} \bracket{\sqrt{2\alpha \gamma(1-\gamma)\kappa_k}+ (1-\gamma)(1+\alpha \kappa_k)}\\
    &&\leq \frac{1}{\alpha\gamma}  \bracket{2\alpha\gamma + (1-\gamma)(2+2\alpha \kappa_k+\kappa_k)}.
\end{eqnarray*}
Then for $P$ and $P^{-1}$, it holds that
\begin{eqnarray*}
\norm{P}&&\leq \sqrt{4} =2,\\
\norm{P^{-1}} &&\leq 2\max_k \abs{\sqrt{\abs{\z^+_k}^2+1} \sqrt{\abs{\z^-_k}^2+1}(\z^+-\z^-)_k^{-1} }\\
&&\leq 2 \max_k \bracket{\abs{(\z^{+}\z^{-})_k} + \abs{\z^{+}_k}+\abs{\z^{-}_k}+1} \abs{(\z^{+}-\z^{-})_k^{-1}}\\
    &&\leq 2 \max_k \bracket{\frac{(1-\gamma)\kappa_k}{\alpha\gamma} + \frac{1}{\alpha\gamma}  \bracket{2\alpha\gamma + (1-\gamma)(2+2\alpha \kappa_k+\kappa_k)}+1 }\frac{\alpha\gamma}{\sqrt{\Delta}}\\
    &&\leq \frac{2}{\sqrt{\Delta}} \bracket{2(1-\gamma)(1+\alpha L+ L)+3\alpha\gamma},
\end{eqnarray*}
where $\Delta = \min_k\bbracket{\abs{\bracket{\gamma+1-\alpha(1-\gamma)\kappa_{k}}^2-4\gamma}}$.

Put 
\begin{eqnarray*}
    \tilde{P}=\left(\begin{array}{cc}
U, & 0_d \\
0_d, & U
\end{array}\right) P.
\end{eqnarray*}
Then we have 
\begin{eqnarray*}
    \norm{\Gamma^j}  = \norm{\Tilde{P} \Lambda^j \Tilde{P}^{-1}} \leq \norm{P}\norm{P^{-1}}\lambda^j \leq \frac{4}{\sqrt{\Delta}}\bracket{2(1-\gamma)(1+\alpha L+ L)+3\alpha\gamma}\lambda^j:=M\lambda^j.
\end{eqnarray*}
\end{proof}

After diagonalizing $\Gamma$, we can obtain an upper bound for the spectral radius $\lambda$ by evaluating the maximal absolute eigenvalue $\max_k\abs{\lambda^{\pm}_{k}}$ in the diagonal matrix $\Lambda$, as specified in definition (\ref{def:Lambda}).

\begin{proof}[Theorem \ref{thm:lambda}] From Lemma \ref{lemma:g_gamma}, the eigenvalues satisfy
\begin{eqnarray*}
    \lambda^{+}_{k} + \lambda^{-}_{k} = \gamma+1-\alpha(1-\gamma)\kappa_k\in [\gamma+1-\alpha(1-\gamma)L, \gamma+1-\alpha(1-\gamma)\mu],
\end{eqnarray*}
For real eigenvalues, there exists $k$ such that $\abs{\gamma+1-\alpha(1-\gamma)\kappa_k}>2\sqrt{\gamma}$. By the fact that $\mu\leq \kappa_k\leq L$, we have 
\begin{eqnarray}
    \notag \lambda = \max_k\abs{\lambda^{\pm}_k} &&= \max\left\{\frac{\gamma+1-\alpha(1-\gamma)\mu + \sqrt{\bracket{\gamma+1-\alpha(1-\gamma)\mu}^2-4\gamma}}{2},\right.\\
    \notag&&\quad\left. \frac{\alpha(1-\gamma)L-\gamma-1 + \sqrt{\bracket{\alpha(1-\gamma)L-\gamma-1}^2-4\gamma}}{2}\right\}\\
    \label{eq:lam}&&=\frac{\gamma+1-(1-\gamma)\phi + \sqrt{\bracket{\gamma+1-(1-\gamma)\phi}^2-4\gamma}}{2}:=h(\gamma),
\end{eqnarray}
where 
\begin{eqnarray*}
    \phi = \min\bbracket{\alpha\mu, \frac{2(1+\gamma)}{1-\gamma} - \alpha L}.
\end{eqnarray*}
In this condition, there holds that $\gamma+1-(1-\gamma)\phi>2\sqrt{\gamma}$.  Define $\eta = \bracket{\gamma+1-(1-\gamma)\phi}^2-4\gamma>0$, then
\begin{eqnarray*}
    h'(\gamma) = \frac{(1+\phi)h(\gamma)-1}{\sqrt{\eta}},\quad h''(\gamma)  = \frac{h'(\gamma)\bracket{(1+\phi)\sqrt{\eta} + 2-(1+\phi)(\gamma+1-(1-\gamma)\phi)}}{2\eta}.
\end{eqnarray*}
By the fact that $h(0)=1-\phi, h'(0)<0$ and $2-(1+\phi)(\gamma+1-(1-\gamma)\phi)\geq 0$, where $0\leq \gamma < (1-\phi)^2/(1+\phi)^2$, we have that $h$ is concave and 
\begin{eqnarray*}
     1-\phi- \frac{\phi(1+\phi)}{1-\phi} \gamma\leq \lambda\leq 1-\phi - \frac{\phi^2}{1-\phi}\gamma.
\end{eqnarray*}

For complex eigenvalues, which means that  $\abs{\gamma+1-\alpha(1-\gamma)\kappa_k}\leq 2\sqrt{\gamma}$ for all $k$. Then $\gamma\geq  (1-\phi)^2/(1+\phi)^2$ and 
\begin{eqnarray*}
    \lambda = \max_k\abs{\lambda^{\pm}_k} = \max_k\bbracket{\abs{\frac{\gamma+1-\alpha(1-\gamma)\kappa_k\pm \sqrt{-1}\sqrt{4\gamma-(\alpha(1-\gamma)\kappa_k-\gamma-1)^{2}}}{2}}}=\sqrt{\gamma}.
\end{eqnarray*}
\end{proof}

\subsection{Proof of Theorem \ref{thm:lrate}}\label{se:rateL>0}

Recall the iteration (\ref{iter}) that
\begin{eqnarray}
    \notag\left(
    \begin{array}{ccc}
      \widetilde{m}_{t+1} \\
      \widetilde{x}_{t+1} \\
    \end{array}
\right) =\Gamma^{t}\left(
    \begin{array}{ccc}
      \widetilde{m}_{1} \\
      \widetilde{x}_{1} \\
    \end{array}
\right)-\alpha (1-\gamma) \sum_{j=1}^t \Gamma^{t-j}\left(
    \begin{array}{ccc}
    0 \\
      \sum_{k=1}^{j}\gamma^{j-k}(\triangledown g_{\eta_{k}}(x_k) -\Sigma\widetilde{x}_k)\\
    \end{array}
\right).
\end{eqnarray}
Define the events
\begin{eqnarray*}
E_{j}=\bbracket{\norm{\left(
    \begin{array}{ccc}
      \widetilde{m}_{j} \\
      \widetilde{x}_{j} \\
    \end{array}
\right)}\leq \frac{C_1}{\sqrt{\batch(1-\lambda)}}\alpha\sigma+C_{2}q_{1}^{1/2}\lambda^{(1-\delta)(j-1)}},\quad 1\leq j\leq t.
\end{eqnarray*}
For certain constants $C_1,C_2 \geq 1$, which will be defined subsequently, the probability of event $E_1$ is $\mP(E_1) = 1$. 
In the remaining part of proof, we use the inductive method to demonstrate that, for any $t\geq 1$, on the events $\cap_{j=1}^{t}E_{j}$, $t\geq 1$, the event $E_{t+1}$ occurs with at least $1 - 2/T^2$ probability.

For $\overline{L}>0$, it holds that
\begin{eqnarray}
    \notag\triangledown g_{\eta_{k}}(x_k) -\Sigma\widetilde{x}_k &&= \bracket{\triangledown g_{\eta_{k}}(x_k)  - \triangledown g_{\eta_{k}}(x^*) +    \triangledown g(x^*)- \triangledown g(x_k) }+ \triangledown g_{\eta_{k}}(x^*) + \triangledown g(x_k)-\Sigma\widetilde{x}_k\\
    \label{l1comp}&&= \bracket{\triangledown g_{\eta_{k}}(x_k)  - \triangledown g_{\eta_{k}}(x^*) +    \triangledown g(x^*)- \triangledown g(x_k) }+ \triangledown g_{\eta_{k}}(x^*) + V_{k}\widetilde{x}_{k},
\end{eqnarray}
where $\norm{V_{k}}=\norm{\int_{0}^{1}\Sigma(x^{*}+y\widetilde{x}_{k})-\Sigma(x^*) dy}\leq \overline{L}\norm{\widetilde{x}_{k}}$.
By applying Lemmas \ref{l1sum1}, \ref{l1sum2} and \ref{l1sum3}, we establish bounds for each of the three components in (\ref{l1comp}).
Then from the iteration (\ref{iter}), with at least $1-2/T^2$ probability, we have that on the events $\cap_{j=1}^{t}E_{j}$
\begin{eqnarray*}
    &&\norm{\left(
    \begin{array}{ccc}
      \widetilde{m}_{t+1} \\
      \widetilde{x}_{t+1} \\
    \end{array}
\right)}\leq M \lambda^t q_1^{1/2} + \alpha\frac{ \sigma}{\sqrt{\batch} }M\frac{\sqrt{2}c\bracket{2\log T+4d}}{(1-\lambda)^{1/2}}+  \alpha M \overline{L}\bracket{\frac{2C_1^2\alpha^2\sigma^2}{\batch(1-\lambda)^2} + \frac{4C_2^2q_1\lambda^{(1-\delta)(t-1)}}{\delta(1-\lambda)}}\\
&&+  \alpha\frac{L_f}{\sqrt{\batch}} M c  (2\log T+4d)  \bracket{\frac{4 C_1 \alpha \sigma }{\sqrt{\batch}(1-\lambda)}+ \frac{4\sqrt{2} C_2 q_1^{1/2}\lambda^{(1-\delta)(t-1)}}{\sqrt{\delta}(1-\lambda)^{1/2}}}\\
&&\leq \frac{C_1}{\sqrt{\batch(1-\lambda)}} \alpha \sigma M \bracket{ \frac{\sqrt{2}c\bracket{2\log T+4d}}{C_1}  +  c  (2\log T+4d) \frac{4\alpha L_f }{\sqrt{\batch}(1-\lambda)^{1/2}} + \frac{2 \overline{L}C_1 \alpha^2 \sigma}{\sqrt{\batch}(1-\lambda)^{3/2}}}\\
&&+ C_2 q_1^{1/2} \lambda^{(1-\delta)t} M\bracket{\frac{\lambda^{\delta t}}{C_2}+  \frac{4\sqrt{2}c  (2\log T+4d) \alpha L_f}{\sqrt{\batch}\lambda^{(1-\delta)}\sqrt{\delta}(1-\lambda)^{1/2}} + \frac{4\overline{L}\alpha C_2q_1^{1/2} }{\lambda^{(1-\delta)}\delta(1-\lambda)}},
\end{eqnarray*}
where $c>0$ is an absolute constant. To ensure the occurrence of event $E_{t+1}$, the learning rate $\alpha$ and the batch size $\batch$ must fulfill the condition
\begin{eqnarray*}
    M\bracket{\frac{4\sqrt{2}c  (2\log T+4d) }{\lambda^{(1-\delta)}\sqrt{\delta}(1-\lambda)^{1/2}}\frac{\alpha L_f}{\sqrt{\batch}} + \frac{2 \overline{L}C_1 \alpha^2 \sigma}{\sqrt{\batch}(1-\lambda)^{3/2}}}\leq \frac{1}{3},
\end{eqnarray*}
where
\begin{eqnarray*}
    C_1 =  3\sqrt{2}c \bracket{2\log T+4d} M .
\end{eqnarray*}
Furthermore, we require that the initialization $q_1= \norm{\widetilde{m}_1}^2+\norm{\widetilde{x}_1}^2=\norm{\widetilde{x}_1}^2$ satisfy
\begin{eqnarray*}
    4M\overline{L}\alpha C_2q_1^{1/2} \frac{1}{\lambda^{1-\delta}\delta(1-\lambda)} \leq \frac{1}{3},
\end{eqnarray*}
where
\begin{eqnarray*}
   C_2 = 3M.
\end{eqnarray*}
It is straightforward to verify that 
\begin{eqnarray*}
    \mP\bracket{E_{t+1}\mid \cap_{j=1}^{t}E_{j}} = \mP\bracket{\norm{\left(
    \begin{array}{ccc}
      \widetilde{m}_{t+1} \\
      \widetilde{x}_{t+1} \\
    \end{array}
\right)}\leq \frac{C_1}{\sqrt{\batch(1-\lambda)}} \alpha \sigma+C_2q_1^{1/2} \lambda^{(1-\delta)t}\mid \cap_{j=1}^{t}E_{j}}\geq 1-\frac{2}{T^2},
\end{eqnarray*}
and the induction is completed. Consequently, we deduce that the probability of the intersection of events from $E_1$ through $E_T$ is at least $\mP\bracket{\bigcap_{j=1}^{T}E_{j}} \geq 1 - \frac{2}{T}$.

\subsection{Supporting Lemmas for Bounds in Theorem \ref{thm:lrate} Proof}

\begin{lemma}\label{l1sum1} Under (A1)-(A2) and (A3'), we have 
    \begin{eqnarray*}
        \mP\bracket{\norm{\sum_{j=1}^{t}\Gamma^{t-j}\left(
    \begin{array}{ccc}
     0 \\
      \sum_{k=1}^j\gamma^{j-k}\triangledown g_{\eta_{k}}(x^{*}) \\
    \end{array}
\right)}\geq M\frac{\sigma}{\sqrt{\batch}}\frac{\sqrt{2}c\bracket{2\log T+4d}}{(1-\lambda)^{1/2}(1-\gamma)}}\leq  \frac{1}{T^2},
    \end{eqnarray*}
    for $1\leq t\leq T$, where $c>0$ is an absolute constant.
\end{lemma}

\begin{proof}
Define
\begin{eqnarray*}
    Y_{t,k} = \sum_{j=k}^t\Gamma^{t-j} \gamma^{j-k}\left(
    \begin{array}{ccc}
     0 \\
    \triangledown g_{\eta_k}(x^*)  \\
    \end{array}
\right).
\end{eqnarray*}
By the definition of the sub-exponential, $Y_{t,k}$ is $\norm{\sum_{j=k}^t\Gamma^{t-j} \gamma^{j-k}}\frac{c\sigma}{\sqrt{\batch}}$-sub-exponential random vector for some absolute constant $c>0$, and $\{Y_{t,k}\}_{k=1}^t$ are independent.
Then by Lemmas \ref{le:sum} and \ref{lemma:g_gamma}, we have
\begin{eqnarray*}
    \sum_{k=1}^t\norm{\sum_{j=k}^t\Gamma^{t-j} \gamma^{j-k}}^2\leq  M^2  \frac{2}{(1-\lambda)(1-\gamma)^2}.
\end{eqnarray*}
By applying Lemma \ref{le:sube} on $\{Y_{t,k}\}_{k=1}^t$, with probability $1-1/T^2$, we can obtain
\begin{eqnarray*}
    \norm{\sum_{k=1}^t Y_{t,k}} \leq  M\frac{\sigma}{\sqrt{\batch}}\frac{\sqrt{2}c\bracket{2\log T+4d}}{(1-\lambda)^{1/2}(1-\gamma)}.
\end{eqnarray*}
\end{proof}

\begin{lemma} \label{l1sum2}
Under (A1)-(A2) and (A3'), suppose that the $t$-th iteration $(\widetilde m_t,\widetilde x_t)$ satisfies
\begin{eqnarray*}
    \sqrt{\norm{\widetilde{m}_{j}}^2+ \norm{\widetilde{x}_{j}}^2}\leq \frac{C_1}{\sqrt{\batch(1-\lambda)}} \alpha\sigma+C_2 q_1^{1/2}\lambda^{(1-\delta)(j-1)},
\end{eqnarray*}
for $1\leq j\leq t$, we have
\begin{eqnarray*}
\mP\bracket{\norm{\sum_{j=1}^{t}\Gamma^{t-j}\left(
    \begin{array}{ccc}
      0 \\
     \sum_{k=1}^j\gamma^{j-k}\bracket{\triangledown g_{\eta_{k}}(x_k)  - \triangledown g_{\eta_{k}}(x^*) +    \triangledown g(x^*)- \triangledown g(x_k) }\\
    \end{array}
\right)}\geq C }\leq  \frac{1}{T^2},
\end{eqnarray*}
where
\begin{eqnarray*}
    C=  c (2\log T+4d)  M \frac{L_f}{\sqrt{\batch}}\bracket{\frac{4 C_1 \alpha \sigma }{\sqrt{\batch}(1-\lambda)(1-\gamma)}+ \frac{4\sqrt{2} C_2 q_1^{1/2}\lambda^{(1-\delta)(t-1)}}{\sqrt{\delta}(1-\lambda)^{1/2}(1-\gamma)}},
\end{eqnarray*}
and $c>0$ is an absolute constant.
\end{lemma}

\begin{proof}
Define
\begin{eqnarray*}
    Y_{t,k} = \sum_{j=k}^t\Gamma^{t-j} \gamma^{j-k}\left(
    \begin{array}{ccc}
     0 \\
    \triangledown g_{\eta_{k}}(x_k)  - \triangledown g_{\eta_{k}}(x^*) +    \triangledown g(x^*)- \triangledown g(x_k)  \\
    \end{array}
\right).
\end{eqnarray*}
According to the iteration of SGDM, we have $\xi_k$ and $x_k$ are independent and $\{Y_{t,k}\}_{k=1}^t$ are martingale difference. Due to the $L_f$-smooth of the individual gradient $\triangledown f_\xi(x)$, it holds that
\begin{eqnarray*}
    \max_{\norm{v}=1}\mE\Bracket{\exp\bracket{\abs{v^\top Y_{t,k}}/\sigma_k}} \leq 2,
\end{eqnarray*}
where 
\begin{eqnarray*}
    \sigma_k = c\frac{L_f}{\sqrt{\batch}}\norm{\sum_{j=k}^t\Gamma^{t-j} \gamma^{j-k}}  \norm{\widetilde{x}_k},
\end{eqnarray*}
and $c>0$ is an absolute constant.
Then according to Lemmas \ref{le:sum} and \ref{lemma:g_gamma}, we have
\begin{eqnarray*}
\sum_{k=1}^t \norm{\sum_{j=k}^t\Gamma^{t-j} \gamma^{j-k}}^2 \norm{\widetilde{x}_k}^2 \leq M^2 \bracket{2C_1^2\alpha^2\frac{\sigma^2}{\batch} \frac{2}{(1-\gamma)^2(1-\lambda)^2} + 2C_2^2 q_1  \frac{4\lambda^{2(1-\delta)(t-1)}}{\delta (1-\gamma)^2(1-\lambda)}}.
\end{eqnarray*}
By applying Lemma \ref{le:sube} on $\{Y_{t,k}\}_{k=1}^t$, with probability $1-1/T^2$, we can obtain the following result:
\begin{eqnarray*}
    \norm{\sum_{j=1}^t Y_{t,j}}  \leq  c (2\log T+4d)  M \frac{L_f}{\sqrt{\batch}}\bracket{\frac{4 C_1 \alpha \sigma }{\sqrt{\batch}(1-\lambda)(1-\gamma)}+ \frac{4\sqrt{2} C_2 q_1^{1/2}\lambda^{(1-\delta)(t-1)}}{\sqrt{\delta}(1-\lambda)^{1/2}(1-\gamma)}}.
\end{eqnarray*}
\end{proof}

\begin{lemma} \label{l1sum3}
Under (A1)-(A2) and (A3'), suppose that the $t$-th iteration $(\widetilde m_t,\widetilde x_t)$  satisfies
\begin{eqnarray*}
    \sqrt{\norm{\widetilde{m}_{j}}^2+ \norm{\widetilde{x}_{j}}^2}\leq \frac{C_1}{\sqrt{\batch(1-\lambda)}} \alpha\sigma+C_2 q_1^{1/2}\lambda^{(1-\delta)(j-1)},
\end{eqnarray*}
for $1\leq j\leq t$, for $\overline{L}>0$, we have
\begin{eqnarray*}
\norm{\sum_{j=1}^{t}\Gamma^{t-j}\left(
    \begin{array}{ccc}
    0 \\
     \sum_{k=1}^j \gamma^{j-k} V_{k}\widetilde{x}_{k} \\
    \end{array}
\right)}\leq  M \overline{L}\bracket{\frac{2C_1^2\alpha^2\sigma^2}{\batch(1-\lambda)^2(1-\gamma)} + \frac{4C_2^2q_1\lambda^{(1-\delta)(t-1)}}{\delta(1-\lambda)(1-\gamma)}}.
\end{eqnarray*}
 \end{lemma}

\begin{proof}
 By the fact that $\sum_{j=1}^t \sum_{k=1}^j \lambda^{t-j}\gamma^{j-k}\leq (1-\lambda)^{-1} (1-\gamma)^{-1}$ and Lemma \ref{le:sum}, we have
\begin{eqnarray*}
&&\norm{\sum_{j=1}^{t}\Gamma^{t-j}\left(
    \begin{array}{ccc}
    0 \\
     \sum_{k=1}^j \gamma^{j-k} V_{k}\widetilde{x}_{k} \\
    \end{array}
\right)}\leq \sum_{j=1}^t M \lambda^{t-j} \sum_{k=1}^j \gamma^{j-k} \overline{L}  \norm{\widetilde{x}_k}^2\cr
&&\leq M \overline{L}  \sum_{j=1}^t\lambda^{t-j}\sum_{k=1}^j \gamma^{j-k} \bracket{\frac{2C_1^2}{\batch(1-\lambda)}\alpha^2\sigma^2 + 2C_2^2q_1\lambda^{2(1-\delta)(k-1)}}\cr
&& \leq  M \overline{L}\bracket{\frac{2C_1^2\alpha^2\sigma^2}{\batch(1-\lambda)^2(1-\gamma)} + 2C_2^2q_1 \sum_{k=1}^t \sum_{j=k}^t \lambda^{t-j}\gamma^{j-k} \lambda^{(1-\delta)(k-1)}}\cr
&&\leq  M \overline{L}\bracket{\frac{2C_1^2\alpha^2\sigma^2}{\batch(1-\lambda)^2(1-\gamma)} + 2C_2^2q_1\frac{2\lambda^{(1-\delta)(t-1)}}{\delta(1-\gamma)(1-\lambda)}}.
\end{eqnarray*}
\end{proof}

\section{Proofs of convergence rates of averaged SGDM}
\subsection{Proof of Theorem \ref{th2}}

By the definition of $\Gamma$, one can easily check that
\begin{eqnarray*}
 \alpha (I-\Gamma)^{-1}\left(
    \begin{array}{ccc}
     0 \\
      \triangledown g_{\eta_{t}}(x^{*}) \\
    \end{array}
\right)=\alpha\left(\begin{array}{cc}
I, & \frac{1}{\alpha} I \\
\frac{\gamma}{\gamma-1} \Sigma^{-1}, & \frac{1}{\alpha} \Sigma^{-1}
\end{array}\right)\left(
    \begin{array}{ccc}
     0 \\
      \triangledown g_{\eta_{t}}(x^{*}) \\
    \end{array}
\right) = \left(
    \begin{array}{ccc}
       \triangledown g_{\eta_{t}}(x^{*})\\
      \Sigma^{-1}\triangledown g_{\eta_{t}}(x^{*}) \\
    \end{array}
\right),
\end{eqnarray*}
and we have
\begin{eqnarray*}
    &&(I-\Gamma)^{-1} \sum_{t=1}^n \left(
    \begin{array}{ccc}
     0 \\
      \triangledown g_{\eta_{t}}(x^{*}) \\
    \end{array}
\right) -(1-\gamma)\sum_{t=1}^n\sum_{j=1}^{t}\Gamma^{t-j}\left(
    \begin{array}{ccc}
     0 \\
      \sum_{k=1}^j\gamma^{j-k}\triangledown g_{\eta_{k}}(x^{*}) \\
    \end{array}
\right)\\
&&=\sum_{k=1}^n \bracket{(I-\Gamma)^{-1}- (1-\gamma)\sum_{t=k}^n\sum_{j=k}^t \Gamma^{t-j} \gamma^{j-k}} \left(
    \begin{array}{ccc}
     0 \\
      \triangledown g_{\eta_{k}}(x^{*}) \\
    \end{array}
\right).
\end{eqnarray*}
Then summing (\ref{iter}) over $t$ from $n_0+1$ to $n$, for $\overline{L}=0$, we obtain
\begin{eqnarray}
\notag &&\sum_{t=n_0+1}^n  \left(
    \begin{array}{ccc}
      \widetilde{m}_{t+1} \\
      \widetilde{x}_{t+1} \\
    \end{array}
\right) +  \sum_{t=n_0+1}^{n}\left(
    \begin{array}{ccc}
     \triangledown  g_{\eta_{t}}(x^{*})  \\
       \Sigma^{-1}\triangledown  g_{\eta_{t}}(x^{*}) \\
    \end{array}
\right) =\sum_{t=n_0+1}^n  \left(
    \begin{array}{ccc}
      \widetilde{m}_{t+1} \\
      \widetilde{x}_{t+1} \\
    \end{array}
\right) + \alpha \sum_{t=n_0+1}^{n}(I-\Gamma)^{-1}\left(
    \begin{array}{ccc}
     0 \\
       \triangledown  g_{\eta_{t}}(x^{*}) \\
    \end{array}
\right) \\
\notag&&=- \alpha (1-\gamma)\sum_{t=1}^n  \sum_{j=1}^t \Gamma^{t-j}\left(
    \begin{array}{ccc}
    0 \\
      \sum_{k=1}^{j}\gamma^{j-k} \bracket{\triangledown g_{\eta_{k}}(x_k) - \triangledown g_{\eta_{k}}(x^*) + \triangledown g(x^*) - \triangledown g(x_k)}\\
    \end{array}
\right)\\
\notag&&+ \alpha (1-\gamma)\sum_{t=1}^{n_0}  \sum_{j=1}^t \Gamma^{t-j}\left(
    \begin{array}{ccc}
    0 \\
      \sum_{k=1}^{j}\gamma^{j-k} \bracket{\triangledown g_{\eta_{k}}(x_k) - \triangledown g_{\eta_{k}}(x^*) + \triangledown g(x^*) - \triangledown g(x_k)}\\
    \end{array}
\right)\\
\notag &&+\alpha \sum_{k=1}^{n}\bracket{(I-\Gamma)^{-1}- (1-\gamma)\sum_{t=k}^n\sum_{j=k}^t \Gamma^{t-j} \gamma^{j-k}} \left(
    \begin{array}{ccc}
     0 \\
      \triangledown g_{\eta_{k}}(x^{*}) \\
    \end{array}
\right)\\
\label{l0avesum}&& -\alpha \sum_{k=1}^{n_0}\bracket{(I-\Gamma)^{-1}- (1-\gamma)\sum_{t=k}^{n_0}\sum_{j=k}^t \Gamma^{t-j} \gamma^{j-k}} \left(
    \begin{array}{ccc}
     0 \\
      \triangledown g_{\eta_{k}}(x^{*}) \\
    \end{array}
\right)+\sum_{t=n_0+1}^{n}\Gamma^{t}\left(
    \begin{array}{ccc}
      \widetilde{m}_{1} \\
      \widetilde{x}_{1} \\
    \end{array}
\right).
\end{eqnarray}
The subsequent proof aims to calculate the expected values of various components within the equation (\ref{l0avesum}). For the $t$-th iteration $(\widetilde m_t,\widetilde x_t)$ satisfying 
\begin{eqnarray*}
    \mE[\norm{\widetilde{m}_t}^2+\norm{\widetilde{x}_t}^2]\leq \frac{C_1}{\batch(1-\lambda)}  \alpha^2\sigma^2 + C_2q_1\lambda^{2(1-\delta)(t-1)}, \quad t\geq 1,
\end{eqnarray*}
by the fact that $ \sum_{t=k}^n \sum_{j=k}^t \lambda^{t-j} \gamma^{j-k} \leq (1-\lambda)^{-1}(1-\gamma)^{-1}$ and $\frac{1-x}{1-x^{\delta}}\leq \delta^{-1}$ for $\delta\in(0,1]$, we have
\begin{eqnarray}
    \notag&&\mE\norm{\sum_{t=1}^n  \alpha (1-\gamma) \sum_{j=1}^t \Gamma^{t-j}\left(
    \begin{array}{ccc}
    0 \\
      \sum_{k=1}^{j}\gamma^{j-k} \bracket{\triangledown g_{\eta_{k}}(x_k) - \triangledown g_{\eta_{k}}(x^*) + \triangledown g(x^*) - \triangledown g(x_k)}\\
    \end{array}
\right)}^2\\
\notag&&\leq \alpha^2 (1-\gamma)^2 \frac{L_f^2}{\batch} M^2 \sum_{k=1}^n \bracket{ \sum_{t=k}^n \sum_{j=k}^t \lambda^{t-j} \gamma^{j-k} }^2 \bracket{\frac{C_1}{\batch(1-\lambda)} \alpha^2\sigma^2 + C_2 q_1 \lambda^{2(1-\delta)(k-1)}}\\
\label{ineq1}&&\leq  \alpha^2  \frac{L_f^2}{\batch} M^2 \bracket{ \frac{C_1\alpha^2\sigma^2}{\batch(1-\lambda)^3} n + \frac{C_2q_1}{(1-\lambda)^2(1-\lambda^{2(1-\delta)})}}\leq \alpha^2  \frac{L_f^2}{\batch} M^2 \bracket{ \frac{C_1\alpha^2\sigma^2}{\batch(1-\lambda)^3} n + \frac{C_2q_1}{(1-\delta)(1-\lambda)^3}}.
\end{eqnarray}
Incorporating Lemmas \ref{le:sum} and \ref{le:sum2}, it is established that
\begin{eqnarray}
    \notag&&\mE\norm{\alpha \sum_{k=1}^{n}\bracket{(I-\Gamma)^{-1}- (1-\gamma)\sum_{t=k}^n\sum_{j=k}^t \Gamma^{t-j} \gamma^{j-k}} \left(
    \begin{array}{ccc}
     0 \\
      \triangledown g_{\eta_{k}}(x^{*}) \\
    \end{array}
\right)}^2\\
\notag&&\leq \alpha^2 \frac{\sigma^2}{\batch} M^2 \sum_{k=1}^n\bracket{2\gamma^2 \bracket{\sum_{j=k}^{n} \lambda^{n-j}\gamma^{j-k}}^2+\frac{2\lambda^{2(n-k+1)}}{(1-\lambda)^2}}\\
\label{ineq2}&&\leq 2\alpha^2 \frac{\sigma^2}{\batch} M^2 \bracket{\frac{2\gamma^2}{(1-\gamma)^2(1-\lambda)}+\frac{\lambda^2}{(1-\lambda)^3}}\leq 6 \alpha^2 \frac{\sigma^2}{\batch} M^2\frac{1}{(1-\lambda)^3},
\end{eqnarray}
where the last inequality is due to the fact that $\sqrt{\gamma}\leq \lambda<1$ in Theorem \ref{thm:lambda}.
Furthermore, according to Lemma \ref{lemma:g_gamma}, it holds that
\begin{eqnarray}
    \label{ineq3}\norm{\sum_{t=n_0+1}^{n}\Gamma^{t}\left(
    \begin{array}{ccc}
      \widetilde{m}_{1} \\
      \widetilde{x}_{1} \\
    \end{array}
\right)}^2\leq M^2 \frac{\lambda^{2n_0}}{(1-\lambda)^2}q_1.
\end{eqnarray}
Combining the above inequalities (\ref{ineq1})(\ref{ineq2})(\ref{ineq3}), we deduce that
\begin{eqnarray*}
    &&(n-n_0)^2\mE\norm{\frac{\sum_{t=n_0+1}^n \widetilde{x}_t}{n-n_0}+\frac{\sum_{k=n_0+1}^n \Sigma^{-1}\triangledown g_{\eta_{k}}(x^{*})}{n-n_0} }^2\\
    &&\leq 5\alpha^2 \frac{L_f^2}{\batch} M^2 \bracket{ \frac{C_1\alpha^2\sigma^2}{\batch(1-\lambda)^3} (n+n_0) + \frac{2C_2q_1}{(1-\delta)(1-\lambda)^3}}\\
    && + 30 \alpha^2 \frac{\sigma^2}{\batch} M^2 \frac{1}{(1-\lambda)^3}+ M^2  \frac{\lambda^{2n_0}}{(1-\lambda)^2}q_1.
\end{eqnarray*}
Take $n_0$ such that $\lambda^{2n_0} = \batch^{-1}(1-\lambda)$ and $n\geq 2n_0$, it holds that
\begin{eqnarray*}
    \mE\norm{\frac{\sum_{t=n_0+1}^n \widetilde{x}_t}{n-n_0}+\frac{ \sum_{k=n_0+1}^n \Sigma^{-1}\triangledown g_{\eta_{k}}(x^{*})}{n-n_0} }^2 \leq  \frac{\tilde{C}_1}{\batch n^2} +  \frac{\tilde{C}_2}{\batch^2 n} ,
\end{eqnarray*}
where
\begin{eqnarray*}
\tilde{C_1} &&= 40 \alpha^2 L_f^2 M^2 \frac{C_2q_1}{(1-\delta)(1-\lambda)^3} + 120 \alpha^2\sigma^2  M^2 \frac{1}{(1-\lambda)^3}+   \frac{4 M^2 q_1}{1-\lambda},\\
\tilde{C_2} &&= 30 \alpha^2\sigma^2 L_f^2  M^2\frac{C_1\alpha^2}{(1-\lambda)^3}.
\end{eqnarray*}




\subsection{Proof of Theorem \ref{th3}}

From (\ref{iter}), for $\overline{L}>0$, we have
\begin{eqnarray*}
    &&\sum_{t=n_0+1}^n  \left(
    \begin{array}{ccc}
      \widetilde{m}_{t+1} \\
      \widetilde{x}_{t+1} \\
    \end{array}
\right) +  \sum_{t=n_0+1}^{n}\left(
    \begin{array}{ccc}
     \triangledown  g_{\eta_{t}}(x^{*})  \\
       \Sigma^{-1}\triangledown  g_{\eta_{t}}(x^{*}) \\
    \end{array}
\right) \\
&&= - \alpha (1-\gamma)\sum_{t=1}^n \sum_{j=1}^t \Gamma^{t-j}\left(
    \begin{array}{ccc}
    0 \\
      \sum_{k=1}^{j}\gamma^{j-k} \bracket{\triangledown g_{\eta_{k}}(x_k) - \triangledown g_{\eta_{k}}(x^*) + \triangledown g(x^*) - \triangledown g(x_k)}\\
    \end{array}
\right)\\
&&+ \alpha (1-\gamma)\sum_{t=1}^{n_0}  \sum_{j=1}^t \Gamma^{t-j}\left(
    \begin{array}{ccc}
    0 \\
      \sum_{k=1}^{j}\gamma^{j-k}\bracket{\triangledown g_{\eta_{k}}(x_k) - \triangledown g_{\eta_{k}}(x^*) + \triangledown g(x^*) - \triangledown g(x_k)}\\
    \end{array}
\right)\\
&&+\alpha \sum_{k=1}^{n}\bracket{(I-\Gamma)^{-1}- (1-\gamma)\sum_{t=k}^n\sum_{j=k}^t \Gamma^{t-j} \gamma^{j-k}} \left(
    \begin{array}{ccc}
     0 \\
      \triangledown g_{\eta_{k}}(x^{*}) \\
    \end{array}
\right)\\
&& -\alpha \sum_{k=1}^{n_0}\bracket{(I-\Gamma)^{-1}- (1-\gamma)\sum_{t=k}^{n_0}\sum_{j=k}^t \Gamma^{t-j} \gamma^{j-k}} \left(
    \begin{array}{ccc}
     0 \\
      \triangledown g_{\eta_{k}}(x^{*}) \\
    \end{array}
\right)\\
&&-\alpha (1-\gamma)\sum_{t=n_0+1}^{n}  \sum_{j=1}^t \Gamma^{t-j}\left(
    \begin{array}{ccc}
    0 \\
      \sum_{k=1}^{j}\gamma^{j-k} V_k\widetilde{x}_k\\
    \end{array}
\right)+\sum_{t=n_0+1}^{n}\Gamma^{t}\left(
    \begin{array}{ccc}
      \widetilde{m}_{1} \\
      \widetilde{x}_{1} \\
    \end{array}
\right).
\end{eqnarray*}
By applying latter Lemmas \ref{le:1}, \ref{le:2} and \ref{le:3}, we establish bounds for each respective component. 
Consequently, with at least probability $1-4/T$, for fixed $\delta\in(0,\frac{1}{2}]$, we have
\begin{eqnarray*}
    &&(n-n_0) \norm{\frac{\sum_{t=n_0+1}^n \widetilde{x}_t}{n-n_0}+\frac{\sum_{k=n_0+1}^n \Sigma^{-1}\triangledown g_{\eta_{k}}(x^{*})}{n-n_0} }\\
    &&\leq c(2\log T+4d)\alpha \frac{L_f}{\sqrt{\batch}} M \bracket{\frac{2\sqrt{2}C_1\alpha\sigma}{\sqrt{\batch}(1-\lambda)^{3/2}}(\sqrt{n}+\sqrt{n_0}) + \frac{8C_2 q_1^{1/2}}{(1-\lambda)^{3/2}}}\\
    &&+2\alpha \frac{\sigma}{\sqrt{\batch}} M \frac{\sqrt{3}c(2\log T+ 4d)}{(1-\lambda)^{3/2}}+\alpha M \overline{L} \bracket{\frac{2C_1^2 \alpha^2\sigma^2}{\batch(1-\lambda)^2} n+ \frac{4 C_2^2q_1\bracket{n_0(1-\lambda)+1}\lambda^{n_0}}{(1-\lambda)^2}} + M\frac{ \lambda^{n_0}}{1-\lambda} q_1^{1/2},
\end{eqnarray*}
where $c>0$ is an absolute constant.
Take $n_0$ such that  $\lambda^{n_0}=\batch^{-1/2}(1-\lambda)^{1/2}$ and $n\geq 2n_0$, then we have 
\begin{eqnarray}\label{l1ave}
    && \norm{\frac{\sum_{t=n_0+1}^n \widetilde{x}_t}{n-n_0}+\frac{\sum_{t=n_0+1}^n \Sigma^{-1}\triangledown g_{\eta_{t}}(x^{*})}{n-n_0} }\leq \frac{\Tilde{C}_1+\Tilde{C}_3}{\sqrt{\batch}n} + \frac{\Tilde{C}_2}{\batch\sqrt{n}} + \frac{\Tilde{C}_4}{\batch},
\end{eqnarray}
where
\begin{eqnarray*}
    \Tilde{C}_1 &&= \alpha L_f M  \frac{16c(2\log T+ 4d) C_2 q_1^{1/2}}{(1-\lambda)^{3/2}}+\alpha \sigma M \frac{4\sqrt{3}c(2\log T+ 4d) }{(1-\lambda)^{3/2}}+  \frac{2M q_1^{1/2}}{(1-\lambda)^{1/2}},\\
    \Tilde{C}_2 &&= \alpha L_f \sigma M  \frac{8\sqrt{2}c(2\log T+4d)C_1 \alpha}{(1-\lambda)^{3/2}} ,\\ \Tilde{C}_3&&= 8 \alpha M \overline{L} \frac{ C_2^2q_1(n_0(1-\lambda)+1)}{(1-\lambda)^{3/2}} ,\quad \Tilde{C}_4 = 4\alpha \sigma^2 M \overline{L} \frac{C_1^2 \alpha^2}{(1-\lambda)^2}.
\end{eqnarray*}
So (\ref{l1ave}) holds for $2n_0\leq n\leq T$ with probability $1-4/T$.

\subsection{Supporting Lemmas for Bounds in Theorem \ref{th3} Proof}

\begin{lemma}\label{le:1} Under (A1)-(A2) and (A3'), suppose that the $t$-th iteration $(\widetilde m_t,\widetilde x_t)$ satisfies
\begin{eqnarray*}
    \sqrt{\norm{\widetilde{m}_{j}}^2+ \norm{\widetilde{x}_{j}}^2}\leq \frac{C_1}{\sqrt{\batch(1-\lambda)}} \alpha\sigma+C_2 q_1^{1/2}\lambda^{(1-\delta)(j-1)},
\end{eqnarray*}
for $1\leq j\leq t$, we have
    \begin{eqnarray*}
        \mP\bracket{\norm{\alpha (1-\gamma)\sum_{t=1}^n  \sum_{j=1}^t \Gamma^{t-j}\left(
    \begin{array}{ccc}
    0 \\
      \sum_{k=1}^{j}\gamma^{j-k} \bracket{\triangledown g_{\eta_{k}}(x_k) - \triangledown g_{\eta_{k}}(x^*) + \triangledown g(x^*) - \triangledown g(x_k)}\\
    \end{array}
\right)}\geq  C}\leq \frac{1}{T^2},
    \end{eqnarray*}
    where
    \begin{eqnarray*}
        C &&= c(2\log T+4d)\alpha \frac{L_f}{\sqrt{\batch}} M \bracket{\frac{2\sqrt{2}C_1\alpha\sigma}{\sqrt{\batch}(1-\lambda)^{3/2}}\sqrt{n} + \frac{2\sqrt{2}C_2 q_1^{1/2}}{(1-\delta)^{1/2}(1-\lambda)^{3/2}}},
    \end{eqnarray*}
    and $c>0$ is an absolute constant.
\end{lemma}
\begin{proof}Define
\begin{eqnarray*}
    Y_{n,k} = (1-\gamma)\sum_{t=k}^n\sum_{j=k}^t \Gamma^{t-j} \gamma^{j-k} \left(
    \begin{array}{ccc}
     0 \\
     \triangledown g_{\eta_{k}}(x_k) - \triangledown g_{\eta_{k}}(x^*) + \triangledown g(x^*) - \triangledown g(x_k) \\
    \end{array}
\right).
\end{eqnarray*}
According to the iteration of SGDM, we have $\xi_k$ and $x_k$ are independent and $\{Y_{t,k}\}_{k=1}^t$ are martingale difference. Due to the $L_f$-smooth of the individual gradient $\triangledown f_\xi(x)$, it holds that
\begin{eqnarray*}
    \max_{\norm{v}=1}\mE\Bracket{\exp\bracket{\abs{v^\top Y_{t,k}}/\sigma_k}} \leq 2,
\end{eqnarray*}
where 
\begin{eqnarray*}
    \sigma_k = c\frac{L_f}{\sqrt{\batch}}(1-\gamma)\norm{\sum_{t=k}^n\sum_{j=k}^t \Gamma^{t-j} \gamma^{j-k}}  \norm{\widetilde{x}_k},
\end{eqnarray*}
and $c>0$ is an absolute constant. By the fact that $\sum_{t=k}^n\sum_{j=k}^t\lambda^{t-j}\gamma^{j-k} \leq (1-\gamma)^{-1} (1-\lambda)^{-1}$ and $\frac{1-x}{1-x^{\delta}}\leq \delta^{-1}$ for $\delta\in(0,1]$, it holds that
\begin{eqnarray*}
    &&\sum_{k=1}^n(1-\gamma)^2  \norm{\sum_{t=k}^n\sum_{j=k}^t \Gamma^{t-j} \gamma^{j-k} }^2\norm{\widetilde{x}_k}^2\\
    &&\leq  M^2 (1-\gamma)^2\sum_{k=1}^n \bracket{\sum_{t=k}^n\sum_{j=k}^t\lambda^{t-j}\gamma^{j-k} }^2\bracket{\frac{2C_1^2}{\batch(1-\lambda)}\alpha^2\sigma^2+2C_2^2q_1 \lambda^{2(1-\delta)(k-1)}}\\
    &&\leq  M^2 \bracket{\frac{2C_1^2\alpha^2\sigma^2}{\batch(1-\lambda)^3}n + \frac{2C_2^2 q_1}{(1-\delta)(1-\lambda)^3}}.
\end{eqnarray*}
By applying Lemma \ref{le:sube} on $\{Y_{n,k}\}_{k=1}^n$, with probability $1-1/T^2$, we have
\begin{eqnarray*}
    \norm{\sum_{k=1}^n Y_{n,k}}\leq c(2\log T+4d)\frac{L_f}{\sqrt{\batch}} M \bracket{\frac{2\sqrt{2}C_1\alpha\sigma}{\sqrt{\batch}(1-\lambda)^{3/2}}\sqrt{n} + \frac{2\sqrt{2}C_2 q_1^{1/2}}{(1-\delta)^{1/2}(1-\lambda)^{3/2}}}.
\end{eqnarray*}
\end{proof}

\begin{lemma}\label{le:2} Under (A1)-(A2) and (A3'), we have
\begin{eqnarray*}
\mP\bracket{\norm{\alpha \sum_{k=1}^{n}\bracket{(I-\Gamma)^{-1}- (1-\gamma)\sum_{t=k}^n\sum_{j=k}^t \Gamma^{t-j} \gamma^{j-k}} \left(
    \begin{array}{ccc}
     0 \\
      \triangledown g_{\eta_{k}}(x^{*}) \\
    \end{array}
\right)}\geq  C}\leq \frac{1}{T^2},
\end{eqnarray*}
where 
\begin{eqnarray*}
    C&&= \alpha \frac{\sigma}{\sqrt{\batch}} M \frac{\sqrt{3}c(2\log T+ 4d)}{(1-\lambda)^{3/2}},
\end{eqnarray*}
and $c>0$ is an absolute constant.
\end{lemma}
\begin{proof}
Define
\begin{eqnarray*}
    Y_{n,k} = \bracket{(I-\Gamma)^{-1}- (1-\gamma)\sum_{t=k}^n\sum_{j=k}^t \Gamma^{t-j} \gamma^{j-k}} \left(
    \begin{array}{ccc}
     0 \\
     \triangledown g_{\eta_k}(x^*)  \\
    \end{array}
\right).
\end{eqnarray*}
Then $\{Y_{n,k}\}_{k=1}^n$ are independent and sub-exponential. By Lemma \ref{le:sum2} and the fact $\sqrt{\gamma}\leq \lambda<1$, it holds that
\begin{eqnarray*}
    \sum_{k=1}^n\norm{(I-\Gamma)^{-1}- (1-\gamma)\sum_{t=k}^n\sum_{j=k}^t \Gamma^{t-j} \gamma^{j-k}}^2&&\leq  M^2\bracket{\frac{2\gamma^2}{(1-\gamma)^2(1-\lambda)}+\frac{\lambda^2}{(1-\lambda)^3}}\leq  \frac{3 M^2}{(1-\lambda)^3}.
\end{eqnarray*}
By applying Lemma \ref{le:sube} on $\{Y_{n,k}\}_{k=1}^n$, with probability $1-1/T^2$, we have
\begin{eqnarray*}
    \norm{\sum_{k=1}^nY_{n,k}}\leq   \frac{\sigma}{\sqrt{\batch}} M\frac{\sqrt{3}c(2\log T+ 4d)}{(1-\lambda)^{3/2}}.
\end{eqnarray*}
\end{proof}


\begin{lemma}\label{le:3}
Under (A1)-(A2) and (A3'), suppose that the $t$-th iteration $(\widetilde m_t,\widetilde x_t)$ satisfies
\begin{eqnarray*}
    \sqrt{\norm{\widetilde{m}_{j}}^2+ \norm{\widetilde{x}_{j}}^2}\leq \frac{C_1}{\sqrt{\batch(1-\lambda)}} \alpha\sigma+C_2 q_1^{1/2}\lambda^{(1-\delta)(j-1)},
\end{eqnarray*}
for $1\leq j\leq t$ and fixed $\delta\in(0,\frac{1}{2}]$, then for $\overline{L}>0$, we have
\begin{eqnarray*}
    &&\norm{\alpha (1-\gamma)\sum_{t=n_0+1}^n\sum_{j=1}^{t}\Gamma^{t-j}\left(
    \begin{array}{ccc}
     0 \\
      -\sum_{k=1}^j\gamma^{j-k}V_k\widetilde{x}_k \\
    \end{array}
\right)}\\
&&\leq \alpha M \overline{L} \bracket{\frac{2C_1^2 \alpha^2\sigma^2}{\batch(1-\lambda)^2} n+4C_2^2q_1 \frac{\bracket{n_0(1-\lambda)+1}\lambda^{n_0}}{(1-\lambda)^2}}.
\end{eqnarray*}
\end{lemma}
\begin{proof}
By the fact that $\sum_{t=k}^n\sum_{j=k}^t\lambda^{t-j}\gamma^{j-k} \leq (1-\gamma)^{-1} (1-\lambda)^{-1}$, we have
\begin{eqnarray*}
    &&\norm{\alpha (1-\gamma)\sum_{t=n_0+1}^n\sum_{j=1}^{t}\Gamma^{t-j}\left(
    \begin{array}{ccc}
     0 \\
      -\sum_{k=1}^j\gamma^{j-k}V_k\widetilde{x}_k \\
    \end{array}
\right)}\\
&&\leq \alpha(1-\gamma) M \overline{L} \sum_{k=1}^n \bracket{\sum_{t=\max\{n_0+1,k\}}^n \sum_{j=k}^t \lambda^{t-j}\gamma^{j-k} }\bracket{\frac{2C_1^2}{\batch(1-\lambda)} \alpha^2\sigma^2+2C_2^2q_1 \lambda^{2(1-\delta)(k-1)}}\\
&&\leq \alpha M \overline{L} \frac{2C_1^2 \alpha^2\sigma^2}{\batch(1-\lambda)^2} n+2\alpha (1-\gamma)M \overline{L} C_2^2q_1  \sum_{k=1}^{n_0}  \bracket{\sum_{t=n_0+1}^n \sum_{j=k}^t \lambda^{t-j}\gamma^{j-k} } \lambda^{2(1-\delta)(k-1)}\\
&& + 2\alpha (1-\gamma)M \overline{L} C_2^2q_1  \sum_{k=n_0+1}^{n}  \bracket{\sum_{t=k}^n \sum_{j=k}^t \lambda^{t-j}\gamma^{j-k} } \lambda^{2(1-\delta)(k-1)}.
\end{eqnarray*}
For the second term, due to that $\lambda\geq \sqrt{\gamma}$ and $\delta\in(0,\frac{1}{2}]$, we have
\begin{eqnarray*}
    &&\sum_{k=1}^{n_0}  \bracket{\sum_{t=n_0+1}^n \sum_{j=k}^t \lambda^{t-j}\gamma^{j-k} } \lambda^{2(1-\delta)(k-1)} \leq \sum_{k=1}^{n_0}  \bracket{ \lambda^{n_0+1-k}\sum_{t=n_0+1}^n \lambda^{t-n_0-1}\sum_{j=k}^t (\gamma/\lambda)^{j-k} } \lambda^{2(1-\delta)(k-1)}\\
    &&\leq \sum_{k=1}^{n_0}  \frac{2}{(1-\lambda)(1-\gamma)}\lambda^{n_0+1-k + 2(1-\delta)(k-1)} \leq \frac{2n_0 \lambda^{n_0}}{(1-\lambda)(1-\gamma)}.
\end{eqnarray*}
For the third term, we have
\begin{eqnarray*}
    \sum_{k=n_0+1}^{n}  \bracket{\sum_{t=k}^n \sum_{j=k}^t \lambda^{t-j}\gamma^{j-k} } \lambda^{2(1-\delta)(k-1)} &&\leq  \sum_{k=n_0+1}^{n}  \frac{1}{(1-\lambda)(1-\gamma)}\lambda^{2(1-\delta)(k-1)} \\
    &&\leq \frac{2\lambda^{n_0}}{(1-\lambda)^2(1-\gamma)}.
\end{eqnarray*}
Thus we draw the conclusion.
\end{proof}



    

\section{Several Useful Lemmas}

\begin{lemma}\label{le:sum}
    For $0\leq \lambda,\gamma<1$, we have
    \begin{eqnarray*}
        \sum_{k=1}^t\bracket{\sum_{j=k}^t \lambda^{t-j}\gamma^{j-k}}^2\leq \frac{2}{(1-\gamma)^2(1-\lambda)}.
    \end{eqnarray*}
    In addition, if $\lambda\geq \sqrt{\gamma}$, for fixed $\delta\in(0,1]$, we have
    \begin{eqnarray*}
        \sum_{k=1}^t\bracket{\sum_{j=k}^t\lambda^{t-j}\gamma^{j-k}}^2 \lambda^{2(1-\delta)(k-1)}\leq  \frac{4}{\delta(1-\gamma)^2(1-\lambda)}\lambda^{2(1-\delta)(t-1)},
    \end{eqnarray*}
    and
    \begin{eqnarray*}
        \sum_{k=1}^t\sum_{j=k}^t\lambda^{t-j}\gamma^{j-k} \lambda^{(1-\delta)(k-1)}\leq  \frac{2}{\delta(1-\gamma)(1-\lambda)}\lambda^{(1-\delta)(t-1)}.
    \end{eqnarray*}
\end{lemma}
\begin{proof}
\begin{eqnarray*}
    \sum_{k=1}^t\bracket{\sum_{j=k}^t \lambda^{t-j}\gamma^{j-k}}^2 &&= \sum_{k=1}^t \sum_{i=k}^t\sum_{j=k}^t\lambda^{2t-i-j}\gamma^{i+j-2k}=\sum_{i=1}^t\sum_{j=1}^t \lambda^{2t-i-j} \sum_{k=1}^{\min\{i,j\}} \gamma^{i+j-2k}\\
    &&\leq \frac{1}{1-\gamma} \sum_{i=1}^t\sum_{j=1}^t \lambda^{2t-i-j} \gamma^{\abs{i-j}}\leq \frac{2}{1-\gamma} \sum_{k=0}^{t-1} \sum_{j=1}^{t-k} \lambda^{2t-2j-k} \gamma^k\\
    &&\leq \frac{2}{1-\gamma} \sum_{k=0}^{t-1} \frac{\lambda^{k}}{1-\lambda^2} \gamma^k\leq \frac{2}{(1-\gamma)(1-\lambda)(1-\lambda \gamma)} \leq \frac{2}{(1-\gamma)^2(1-\lambda)}.
\end{eqnarray*}
If $\lambda\geq \sqrt{\gamma}$, we have
\begin{eqnarray*}
    &&\lambda^{-2(1-\delta)(t-1)}\sum_{k=1}^t\bracket{\sum_{j=k}^t\lambda^{t-j}\gamma^{j-k}}^2 \lambda^{2(1-\delta)(k-1)} = \sum_{k=1}^t \sum_{i=k}^t\sum_{j=k}^t \lambda^{2k-i-j} \gamma^{i+j-2k} \lambda^{2\delta(t-k)}\\
    &&\leq \sum_{i=1}^t \sum_{j=1}^t\lambda^{2\delta(t-\min\{i,j\})} \sum_{k=1}^{\min\{i,j\}}\bracket{\frac{\gamma}{\lambda}}^{i+j-2k}\leq \frac{1}{1-\gamma^2/\lambda^2}  \sum_{i=1}^t \sum_{j=1}^t\lambda^{2\delta(t-\min\{i,j\})} \bracket{\sqrt{\gamma}}^{\abs{i-j}}\\
    &&\leq \frac{2}{1-\gamma} \sum_{i=1}^t \sum_{k=0}^{t-i} \lambda^{2\delta(t-i)} \bracket{\sqrt{\gamma}}^{k}\leq \frac{4}{(1-\gamma)^2(1-\lambda^{2\delta})}\leq \frac{4}{\delta(1-\gamma)^2(1-\lambda)},
\end{eqnarray*}
where the last inequality is due to the fact that $\frac{1-x}{1-x^\delta}\leq \frac{1}{\delta}$ for $\delta\in(0,1]$. Similarly, we have
\begin{eqnarray*}
    &&\lambda^{-(1-\delta)(t-1)}\sum_{k=1}^t\sum_{j=k}^t\lambda^{t-j}\gamma^{j-k} \lambda^{(1-\delta)(k-1)}= \sum_{k=1}^t\sum_{j=k}^t\lambda^{\delta(t-k)}\bracket{\frac{\gamma}{\lambda}}^{j-k} \\
    &&\leq \sum_{j=1}^t \lambda^{\delta(t-j)} \sum_{k=1}^j (\sqrt{\gamma})^{j-k} \leq \frac{2}{1-\gamma} \frac{1}{1-\lambda^\delta} \leq \frac{2}{\delta(1-\gamma)(1-\lambda)}.
\end{eqnarray*}

\end{proof}

\begin{lemma}\label{le:sube}
    Suppose $x_k\in \mR^d$ are independent, zero-mean $\sigma_k$-sub-exponential random vector, we have
    \begin{eqnarray*}
        \mP\bracket{\norm{\sum_{k=1}^t x_k}>c\bracket{\log(1/\epsilon)+2d}\sqrt{\sum_{k=1}^n  \sigma_k^2}}\leq \epsilon,
    \end{eqnarray*}
    where $c>0$ is an absolute constant.
\end{lemma}
\begin{proof}
Let $\mathcal{N}$ be the 1/2-net of the unit ball $\{z\in\mR^d:\norm{z}\leq 1\}$ with respect to the Euclidean norm that satisfies $\abs{\mathcal{N}}\leq 6^d$. By the inequality that 
 \begin{eqnarray*}
     \max_{\norm{v}=1} v^\top x \leq \max_{z\in\mathcal{N}} z^\top x +\max_{\norm{v}=\frac{1}{2}} v^\top x = \max_{z\in\mathcal{N}} z^\top x +\frac{1}{2}\max_{\norm{v}=1} v^\top x,
 \end{eqnarray*}
we have
\begin{eqnarray*}
    \mP[\max_{\norm{v}=1} v^\top \sum_{k=1}^n x_k> \delta] \leq \mP[2\max_{z\in \mathcal{N}} z^\top \sum_{k=1}^n x_k> \delta] \leq 6^d \sup_{\norm{z}=1}\mP[2 z^\top \sum_{k=1}^n x_k> \delta].
\end{eqnarray*}
According to Bernstein’s inequality, there holds
\begin{eqnarray*}
    \mP[2 z^\top \sum_{k=1}^n x_k> \delta] \leq \exp\bracket{-c_0 \min\bbracket{\frac{\delta^2}{4\sum_{k=1}^n  \sigma_k^2},\frac{\delta}{2\max_k \sigma_k}}},
\end{eqnarray*}
for some absolute constant $c_0$.
To bound $\mP[\max_{\norm{v}=1} v^\top \sum_{k=1}^n x_k> \delta] \leq \epsilon$, we find $\delta$ such that
\begin{eqnarray*}
    \delta &&= 2\max\bbracket{\sqrt{\sum_{k=1}^n  \sigma_k^2}\frac{1}{\sqrt{c_0}}\sqrt{\log(1/\epsilon)+2d}, \max_k \sigma_k \frac{1}{c_0}\bracket{\log(1/\epsilon)+2d}}\leq c\sqrt{\sum_{k=1}^n  \sigma_k^2}\bracket{\log(1/\epsilon)+2d},
\end{eqnarray*}
where $c= 2 \max\bbracket{\frac{1}{c_0},\frac{1}{\sqrt{c_0}}}.$
\end{proof}

\begin{lemma}\label{le:sum2}
For $\Gamma$ satisfying the conditions in Lemma \ref{lemma:g_gamma}, we have
    \begin{eqnarray*}
        \norm{(I-\Gamma)^{-1}- (1-\gamma)\sum_{t=k}^n\sum_{j=k}^t \Gamma^{t-j} \gamma^{j-k}}&&\leq M \bracket{\gamma\sum_{j=k}^{n} \lambda^{n-j}\gamma^{j-k}+\frac{\lambda^{n-k+1}}{1-\lambda}},
    \end{eqnarray*}
and
    \begin{eqnarray*}
        \norm{(I-\Gamma)^{-1}- \sum_{t=0}^n \Gamma^t}&&\leq M \frac{\lambda^{n+1}}{1-\lambda}.
    \end{eqnarray*}
\end{lemma}
\begin{proof}
From Lemma \ref{lemma:g_gamma}, we have $\Gamma = \Tilde{P}\Lambda \Tilde{P}^{-1}$. Then
\begin{eqnarray*}
     &&\norm{(I-\Gamma)^{-1}- (1-\gamma)\sum_{t=k}^n\sum_{j=k}^t \Gamma^{t-j} \gamma^{j-k}} \leq \norm{P}\norm{P^{-1}} \norm{(I-\Lambda)^{-1}- (1-\gamma)\sum_{t=k}^n\sum_{j=k}^t \Lambda^{t-j} \gamma^{j-k}}\\
     &&\leq M \norm{\sum_{j=0}^\infty \Lambda^j- (1-\gamma)\sum_{j=0}^{n-k} \Lambda^j \sum_{l=0}^{n-k-j}\gamma^l} =  M \norm{\sum_{j=0}^{n-k} \Lambda^j (1-(1-\gamma)\sum_{l=0}^{n-k-j}\gamma^l)+\sum_{j=n-k+1}^\infty \Lambda^j}\\
     &&\leq M \bracket{\gamma\sum_{j=k}^{n} \lambda^{n-j}\gamma^{j-k}+\frac{\lambda^{n-k+1}}{1-\lambda}}.
\end{eqnarray*}
Also we have
\begin{eqnarray*}
    \norm{(I-\Gamma)^{-1}- \sum_{t=0}^n \Gamma^t}\leq M \sum_{j=n+1}^\infty \lambda^j \leq M\frac{\lambda^{n+1}}{1-\lambda}.
\end{eqnarray*}
\end{proof}
\end{appendices}

\end{document}